\documentclass[twoside, 11pt]{article}

\usepackage[abbrvbib, preprint]{jmlr2e}

\usepackage{amsmath,amsfonts,amssymb,mathtools}
\usepackage{bm}

\usepackage[utf8]{inputenc}
\usepackage[T1]{fontenc}
\usepackage{microtype}

\usepackage{bookmark}
\usepackage{booktabs}

\usepackage{float}
\usepackage{graphicx}

\usepackage[dvipsnames]{xcolor}
\usepackage{booktabs}

\usepackage{tikz}
\usetikzlibrary{arrows.meta,positioning,calc}
\usepackage{pgfplots}
\pgfplotsset{width=10cm,compat=1.9}

\usepgfplotslibrary{external}

\usepackage{algorithm}
\usepackage{algorithmic}

\DeclareMathOperator{\Tr}{Tr}

\newcommand{\RR}{\mathbb{R}}

\newcommand{\NN}{\mathbb{N}}

\newcommand{\XX}{\mathbb{X}}
\newcommand{\UU}{\mathbb{U}}

\newcommand{\EE}{\mathbb{E}}
\newcommand{\PP}{\mathbb{P}}

\newcommand{\sZ}{\mathsf{Z}}
\newcommand{\sY}{\mathsf{Y}}

\newcommand{\cB}{\mathcal{B}}
\newcommand{\cC}{\mathcal{C}}
\newcommand{\cD}{\mathcal{D}}

\newcommand{\cG}{\mathcal{G}}

\newcommand{\cI}{\mathcal{I}}

\newcommand{\cK}{\mathcal{K}}

\newcommand{\cN}{\mathcal{N}}

\newcommand{\cO}{\mathcal{O}}
\newcommand{\cP}{\mathcal{P}}

\newcommand{\cT}{\mathcal{T}}

\newcommand{\cZ}{\mathcal{Z}}

\newcommand{\fC}{\mathfrak{C}}

\newcommand{\fc}{\mathfrak{c}}

\def\Id{\mathbf{I}}
\def\zero{\mathbf{0}}

\allowdisplaybreaks

\def\ka{}

\newcommand{\inner}[2]{\langle#1, #2\rangle}

\newcommand{\e}{\mathrm{e}} 
\newcommand{\de}{\mathrm{d}} 

\newcommand{\argmin}{\operatorname{argmin}}

\title{An Optimal Control Approach to Transformer Training}

\author{\name Ka\u{g}an Akman \email kagan.akman@bilkent.edu.tr \\
       \addr Department of Mathematics\\
       Bilkent University\\
       Ankara, 06800, Turkey
       \AND
       \name Naci Saldi \email naci.saldi@bilkent.edu.tr \\
       \addr Department of Mathematics\\
       Bilkent University\\
       Ankara, 06800, Turkey
       \AND
       \name Serdar Y\"{u}ksel \email yuksel@queensu.ca \\
       \addr Department of Mathematics and Statistics\\Queen's University\\Kingston K7L 3N6, ON, Canada}

\editor{}

\jmlrheading{X}{2026}{X-XX}{X/XX}{X/XX}{XX-XXX}{Kağan Akman, Naci Saldı, Serdar Yüksel}
\ShortHeadings{An Optimal Control Approach to Transformer Training}{Akman, Saldı, Yüksel}

\begin{document}

\maketitle

\begin{abstract}%
In this paper, we develop a rigorous optimal control-theoretic approach to Transformer training that respects key structural constraints such as (i) realized-input-independence during execution, (ii) the ensemble control nature of the problem, and (iii) positional dependence. We model the Transformer architecture as a discrete-time controlled particle system with shared actions, exhibiting noise-free controlled McKean--Vlasov dynamics, and we show that lifting the dynamics to probability measures produces a fully-observed Markov decision process (MDP). Positional encodings are incorporated into the state space to preserve the sequence order under lifting. Using the dynamic programming principle, we establish the existence of globally optimal policies under mild assumptions of compactness. We further prove that closed-loop policies in the lifted { problem} is equivalent to an initial-distribution dependent open-loop policy, which are realized-input-independent and compatible with standard Transformer training. To train a Transformer, we propose a triply quantized training procedure for the lifted MDP by quantizing the state space, the space of probability measures, and the action space, and show that any optimal policy for the triply quantized model is near-optimal for the original training problem. Finally, we establish stability and empirical consistency properties of the lifted model by showing that the value function is continuous with respect to the perturbations of the initial empirical measures and convergence of policies as the data size increases. { This approach provides a globally optimal and robust theoretical alternative to gradient-based training without requiring smoothness or convexity.}
\end{abstract}

\section{Introduction}
Transformers, introduced by \cite{vaswani2017attention}, form a significant class of neural networks with a special mechanism called \emph{self-attention}. Many large language models (LLMs), such as GPT-4, use Transformers as their backbones, and they are proven to be very efficient and successful, see \cite{achiam2023gpt}'s technical paper. A Transformer is a back-to-back concatenation of so-called \emph{feed-forward layers} with \emph{self-attention blocks}. Here, a feed-forward layer is a composition of two affine functions with a sigmoidal nonlinear function in between, providing useful approximation properties where some of them can be reviewed from the work of \cite{hornik1989multilayer}. A \emph{self-attention block} is considered as a function processing a sequence of vectors that "relates" them using a kernel described by two matrices called \emph{query} and \emph{key} matrices. This process returns a probability distribution called attention scores and maps each vector to a weighted average of vectors obtained from attention scores and one more matrix called \emph{value} matrix. In applications, a Transformer is usually endowed with \emph{encoder}, \emph{decoder} blocks, and normalization between each pair of layers and blocks, which can be seen in the work of \cite{vaswani2017attention}.

Conventionally, a Transformer, like many classes of neural networks, is trained by means of gradient-descent-based methods, which are found to be very effective in practice. However, the sufficient conditions for gradient descent to find a global minimum requires smoothness and convexity assumptions that neural networks usually do not satisfy.

In this paper, following \cite{geshkovski2025mathematical}'s contribution, we form an optimal control---theoretic framework for Transformers. In particular, we regard Transformers, at the particle level, as a finite-horizon, discrete-time dynamical system exhibiting McKean--Vlasov flow similar to the setup used by \cite{pham2016discrete}; yet in our case, we have shared controls applied to each particle simultaneously and therefore, this formulation can also be viewed as an ensemble control problem as in the work of \cite{li2009ensemble}. Then, we proceed with defining an optimal control problem by specifying a data set, formulating an appropriate information structure, appropriate controls, a cost structure, and the corresponding optimality notion.

We observe that this dynamical system, at the particle level, exhibits a non-Markovian nature since the measure dependence of the dynamics breaks the Markovian property. A solution to this issue is lifting the problem to the space of probability measures, similar to the methods utilized by \cite{bauerle2023mean,carmona2023model,sanjari2026optimality}. However, this lifting to empirical measures results in the loss of positional information. Therefore, before the lifting procedure, we incorporate positional encodings to the state space so that we restore the positional relations between particles.

Then, we pose the corresponding dynamic programming equations for the final MDP and show the existence of a globally optimal solution that comes in the form of a \emph{closed-loop} policy, or in other words, deterministic and Markovian. Exploiting the deterministic flow of Transformers, it is possible to obtain \emph{open-loop} policies that are optimal. Using such policies, we fix the weights and consequently fix the structure of the Transformer. By fixing the weights, we establish consistency with the realized-input-independence during execution.

We propose a quantization based training algorithm for our formulation since the dynamic programming for continuous state-space MDPs is computationally expensive. By quantizing the state space, the measure space on the quantized state space, and the action space, we define a new measure-valued MDP with finite state and action spaces for which the existence of optimal controls are guaranteed. Then, we show that optimal policies for the triply quantized model are near-optimal for the lifted MDP, and therefore, for the original problem. Finally, we provide an illustrative example to demonstrate the performance of optimal policies obtained from the triply quantized problem with respect to the level of quantization.

Beyond existence and near-optimality results, our lifted formulation enjoys stability properties. Owing to the weak Feller structure and a compactness assumption, the resulting value function depends continuously on the underlying empirical measures. This ensures that as the data set becomes a better sampling of the underlying true distribution, the corresponding optimal control actions become near-optimal for the true distribution, which we establish rigorously.

\subsection{Related Work} To the best of our knowledge, this is the first paper that formulates and provides an alternative approach to train Transformers using the theory of Markov decision processes and dynamic programming. Because of its nature, this work stands on the crossroads of many research areas.
\medskip

\noindent\textbf{Control--theory approaches to neural networks.}  An important line of research focuses on interpreting deep neural networks as dynamical systems. A pioneering work in such perspective is done by \cite{weinan2017proposal} where deep neural networks are considered as discretizations of the continuous-time dynamical systems. An approach by regarding deep learning as  a parameter estimation problem for a nonlinear dynamical system is followed in the work of \cite{haber2017stable}, where they also provide stability and well-posedness results. For a PDE-based treatment of convolutional neural networks, see the study of \cite{ruthotto2020deep}. In the study of \cite{chen2018neural}, residual neural networks are seen as ordinary differential equations. A comprehensive study in control of neural ordinary differential equations is provided by \cite{ruiz2023neural}.
\medskip

\noindent\textbf{Optimal control formulations of training neural networks.} \cite{li2018maximum} develop training algorithms for deep residual neural networks using Pontryagin’s maximum principle. The authors show that the resulting algorithm coincides with the continuous-time counterpart of backpropagation (i.e. gradient descent). Consequently, the obtained controls are generally suboptimal, which is natural since the maximum principle provides only necessary conditions for optimality. The connection between deep learning and mean-field problems is further explored by \cite{E2019}, where residual networks are interpreted as continuous-time control systems, they derive the associated Hamilton–Jacobi–Bellman (HJB) equations for this continuous-time control system. A broader overview of optimal control perspectives on neural networks is surveyed by \cite{benning2021}. There are several studies by \cite{chizat2018global, mei2018mean, sirignano2022mean} using measure-lifting techniques to approach deep learning. However, these works all rely on convexifying the control space by lifting to probability measure, whereas our study uses deterministic control.
\medskip

\noindent\textbf{Mean-field control and McKean--Vlasov dynamics.} Due to the attention mechanism in Transformer architectures, the resulting dynamics can naturally be viewed as a McKean--Vlasov type control problem. This differs fundamentally from residual neural networks, since here the empirical distribution of the particles explicitly enters the dynamics. Such distribution-dependent dynamics were studied in the context of stochastic processes induced by stochastic differential equations by \cite{mckean1966class}. From a control-theoretic perspective, McKean--Vlasov control problems —and the closely related mean-field control or mean-field game problems — have been extensively investigated in both continuous-time and discrete-time settings by \cite{bayraktar2018randomized,carmona2015forward,carmona2013control,motte2022mean,pham2017dynamic,pham2016discrete,PY25MckeanVA}. Our primary objective is to build on but also extend this rich body of literature by adapting the machinery to the specific operational constraints of Transformers training, in order to develop algorithms which are optimal in a mathematically precise sense. In addition to optimality, we also arrive at near optimal solutions in an algorithmic and computational sense.
\medskip

\noindent\textbf{Mathematical formulation on Transformer design}. \cite{geshkovski2025mathematical} provide a rigorous formulation of Transformers by modeling them as a continuous-time \emph{uncontrolled} dynamical system of interacting particles. In contrast to our control-theoretic perspective, their focus is on analyzing the intrinsic dynamics and establishing asymptotic clustering properties, following their earlier work \citep{geshkovski2023emergence}. Although control of Transformers is mentioned there as a possible research direction, it is not developed. This direction is precisely what we pursue in the present paper. \cite{yun2019transformers} establish that the transformer class is a universal approximator of continuous functions on compact domains. Using the continuity equation in \citep{geshkovski2025mathematical}, \cite{geshkovski2024measure} view Transformers as interpolators between measures. The framework provided by \cite{geshkovski2025mathematical} is extended to broader classes of Transformer architectures with different self-attention mechanisms, and the well-posedness of the associated continuity equation is analyzed by \cite{castin2025unified}. Since these works do not formulate Transformers within an optimal control framework with shared controls—as we do here—our contribution can be viewed as complementary. 
\medskip

\noindent\textbf{Generalization problem and robustness to incorrect priors.} Recent advances in generalization problem have taken in multiple directions. \cite{xu2017information} provide an information---theoretic approach to generalization problem by controlling the error using the mutual information between the data set and output. \cite{lugosi2022generalization} take a convex-analytic path by replacing the mutual information with a general strongly convex dependence functional of the joint data–output distribution. \cite{mwigo2025generalization} derive a Transformer-specific generalization bound for shallow Transformers trained with gradient descent.

\subsection{Contributions} This paper provides a rigorous and implementation-faithful optimal control--theoretic formulation of Transformer training. Our main contributions can be summarized as follows:
\begin{itemize}
    \item[(i)] \textbf{Control theoretic formulation compatible with Transformers architecture.} We formulate Transformer architectures as discrete-time controlled dynamical systems of interacting particles evolving under shared control actions with the following salient attributes.
    \begin{itemize}
    \item \textbf{Ensemble control with McKean-Vlasov dynamics.} This perspective leads to an ensemble control problem with deterministic McKean--Vlasov dynamics, distinguishing our framework from exchangeable mean-field models and classical single-agent control problems while being equivalent to the original training problem.
\item \textbf{Open-loop design and compatibility with fixed weights in Transformers design.}
    To obtain data-independent coefficients, which are inevitably open-loop, we establish and build on the following novel equivalence properties: a closed-loop policy of a lifted problem is equivalent to an initial-distribution dependent open-loop policy (owing to the deterministic and ensemble nature of the lifted problem), which is, in turn, realized-input-independent and thus a genuine open-loop policy. This procedure corresponds exactly to fixing the weights after the training phase in classical neural network training, and thereby satisfies realized input-independence during execution (Figure \ref{EquivalenceControlPol}). 

    \item \textbf{Explicit treatment of positional dependence and MDP lifting.} We lift the problem to the space of probability measures, which induces a measure-valued MDP with finite-horizon terminal cost criterion equivalent to the particle-level problem. However, measure lifting destroys the positional information. Therefore, we incorporate positional encodings into the state space before the lifting procedure to recover positional information.

     \item \textbf{Existence of globally optimal controls.} We prove that the {dynamics of the lifted problem} have a Markovian transition kernel satisfying weak Feller property under compactness assumptions (Corollary \ref{cor:transition-kernel}). This allows the application of dynamic programming and establishes the existence of optimal controls that are globally optimal for the lifted MDP (Theorem \ref{thm:Existence}).

\item \textbf{Optimality guarantees imply that reachable measures are attainable by training}
Several studies present what transformers can do in terms of exact sequence to sequence matchings with implications on {\it reachability} or {\it controllability}: \cite{yun2019transformers} establish that the transformer class is a universal approximator of continuous functions on compact domains. Using the continuity equation in \citep{geshkovski2025mathematical}, \cite{geshkovski2024measure,castin2025unified} view Transformers as interpolators between measures and \cite{rajaraman2024Transformers} studies representability of conditional Markov measures. However, only with optimally trained control (with guaranteed optimality) under a suitable cost which penalizes deviation from a target, one can argue that if some control exists that attains some level of reachability, an optimization can recover such a control. Since gradient based methods do not claim global optimality or near optimality, our paper thus closes a significant gap. We note the analogy with classical optimal control where under controllability, linear quadratic type cost criteria provides an efficient method for target tracking \cite{anderson2007optimal}.

    \end{itemize}

\begin{figure}[H]
\centering
        \begin{tikzpicture}[
          scale=0.7,
          transform shape,
          node distance=18mm,
          box/.style={draw, rounded corners, align=center, inner sep=6pt, text width=4.2cm},
          eq/.style={font=\Large},
          note/.style={font=\footnotesize, align=center, text width=4.8cm},
        ]
            \node[box] (cl) {Closed-loop policy\\(lifted problem)};
            \node[eq, right=8mm of cl] (e1) {$\cong$};

            \node[box, right=8mm of e1] (ol0) {Initial-distribution\\dependent open-loop policy};
            \node[note, below=6mm of e1] (why) {deterministic flow\\+ ensemble control};

            \node[eq, right=8mm of ol0] (e2) {$\cong$};

            \node[box, right=8mm of e2] (ol) {Realized-input-independent\\open-loop policy};
        \end{tikzpicture}
\caption{Equivalence Properties of Control Policies for Transformer Design.}
\label{EquivalenceControlPol}
\end{figure}

    \item[(ii)] \textbf{Computational design via triply quantized learning scheme.} We propose a training procedure based on quantizing the particle state space, the measure space, and the action space. We show that this formulation gives rise to a finite state, finite action, measure-valued MDP, for which the computation of optimal policies via dynamic programming recursions is tractable. While establishing the existence of optimal policies for the triply quantized measure-valued MDP, we also provide a method for computing such policies (Theorem \ref{thm:triply-quantized-optimal-policies}) via dynamic programming recursions and then, we prove that any optimal policy for the triply quantized case is near-optimal for the original problem (Theorem \ref{thm:triply-quantized-efficiency}).

    \item[(iii)] \textbf{Robustness to training data, empirical consistency and asymptotic $\Gamma$-convergence to optimality.} \ka{By using a double-lifted formulation, which is a generalization of the lifted-formulation used for training, we establish the continuity of the cost functional with respect to the training data (Theorem \ref{thm:robustness}).}
\end{itemize}

\subsection{Conventions and Notation}
{Let $(\Omega,\Sigma)$ be a Borel space, where the $\sigma$-field structure is given by a locally compact Hausdorff topology. Then $\cP(\Omega)$ is the space of all probability measures on $(\Omega,\Sigma)$. We consider it with the weak$^*$-topology induced by the duality with the space of continuous and bounded functions on $\Omega$, denoted $\cC_b(\Omega)$. For compact $\Omega$, $\cC_b(\Omega) = \cC(\Omega)$, i.e. all continuous functions on $\Omega$ are bounded.

\begin{center}
    \begin{tabular}{@{}lll@{}}
        \toprule
        Symbol & Meaning & First used in\\
        \midrule
        $\Pi$ & Normalization operation & p.7 \\[5pt]
        $\cN$ & $\{1,\ldots, N\},\;N\in\NN$ & p.7\\[5pt]
        $\cT$ & $\{0,\ldots, T-1\},\;T\in\NN$ & p.7\\[5pt]
        $\overline{\cT}$ & $\{0,\ldots,T\}$ & p.7 \\[5pt]
        $\cK$ & $\{1,\ldots, K\},\;K\in\NN$ & p.8\\[5pt]
        $a^{1:n}$ & $(a^1,\ldots, a^n)$ for $n\geq 2$ & p.12\\[5pt]
        \bottomrule
    \end{tabular}
\end{center}}

\section{The Proposed Optimal Control Formulation of Transformers}\label{sec:formulation}
In this section, we interpret the forward dynamics of  Transformers as a discrete-time controlled dynamical system of interacting particles and show how the training procedure can be formulated as a finite-horizon optimal control problem.

\cite{geshkovski2025mathematical} describe Transformers via the flow map of a continuous-time dynamical system. We adopt this approach, but, rather than considering the continuous-time system, we study a discrete-time version of the resulting dynamics that models a Transformer block with a feed-forward layer and single-headed self-attention. The formulation presented in this paper can be generalized to the multi-headed self-attention case. We therefore restrict ourselves to the single-headed self-attention for brevity and simplicity of the presentation.

\subsection{Transformer dynamics as a controlled particle system}
{
In this section, a controlled particle system formulation of Transformer dynamics is presented. We commence with defining the embedding matrix $E\in\RR^{(d_p + d)\times d}$ by $E = [\zero_{d_p}\; \Id_d]^\top$ where $\Id_d$ is the identity matrix acting on $\RR^d$ and $\zero_{d_p}$ is the zero matrix on $\RR^{d_p}$, so that $Es = (0,s)\in\RR^{d_p\times d}$ for all $s\in\RR^d$. Define $\Pi := I - EE^\top + E\circ\mathbf{P}\circ E^\top$ where ${\mathbf{P}}:\RR^d\to\RR^d$ is an $L_{\mathbf{P}}$-Lipschitz function satisfying $\mathbf{P}(\RR^d)\subset S$  and $S\subseteq\RR^d$ is compact. Note that
\begin{equation}
    \label{eq:Embeddings}
    (I - EE^\top)(p, s) = (p, 0) \text{ and } (E\circ\mathbf{P}\circ E^\top)(p, s) = E\mathbf{P}(E^\top(p,s)) = E\mathbf{P}(s)
\end{equation}
for all $(p,s)\in\RR^{d_p + d}$, hence $\Pi(p,s) = (p, \mathbf{P}s)$. We restrict ourselves to a finite $\mathrm{PE}_N\subset\RR^{d_p}$ in the $p$-coordinates, which are called {\it positional encodings}. For brevity and convenience, we write $\XX := \mathrm{PE}_N\times S$. Let us define the space $\UU$ of tuples of matrices $U = (W, A, b, Q, K, V)$ where

\begin{center}
    \begin{tabular}{@{}lll@{}}
        \toprule
        Matrix & Dimension \\
        \midrule
        $W$ & $d\times d_1$ \\
        $A$ & $d_1\times (d_p + d)$ \\
        $b$ & $d_1$ \\
        $Q, K$ & $d_2\times(d_p+ d)$\\
        $V$ & $d\times(d_p+ d)$ \\
        \bottomrule
    \end{tabular}
\end{center}
We denote a {\it weight trajectory} by $\underline{U} := (U_t)_{t\in\cT} = ((W_t, A_t, b_t, Q_t, K_t, V_t))_{t\in\cT}\in\UU^T$. We use the name weight and action synonymously throughout this paper.

\begin{definition}\label{def:Transformer}
    A Transformer with weights $\underline{U}\in\UU^T$ is a map from $S^N \to S^N$ carrying each component, denoted by $i\in\cN$, with
    \begin{equation}\label{eq:Dynamics}
        x_{t+1}^i = \Pi\left( x_t^i + E \bigg(W_t\sigma(A_tx_t^i + b_t) + \sum_{j=1}^N \frac{ \e^{\beta\inner{Q_tx_t^i}{K_tx_t^j}}}{\sum_{j^\prime=1}^N\e^{\beta\inner{Q_tx_t^i}{K_tx_t^{j^\prime}}}}V_tx_t^j\bigg) \right)
    \end{equation}
    where $t\in\cT$ and the inputs $(s^1,\ldots,s^N)\in S^N$ are fed into the Transformer by $x_0^i = (p_i, s^i) \in \mathrm{PE}_N\times S$ for each $i\in\cN$. The outputs are recovered by applying $E^\top$
\end{definition}

The part
\[
    z_t^i := W_t\sigma(A_tx_t^i + b_t) + \sum_{j=1}^N \frac{ \e^{\beta\inner{Q_tx_t^i}{K_tx_t^j}}}{\sum_{j^\prime=1}^N\e^{\beta\inner{Q_tx_t^i}{K_tx_t^{j^\prime}}}}V_tx_t^j
\]
yields a vector in $\RR^d$. By applying the embedding matrix $E$ to it, we obtain an augmented state $(0, z_t^i)$. On the other hand, By explicitly applying $\Pi$ and writing $x_t^i = (p_i, s_t^i)$, we obtain
\begin{align*}
    \Pi(x_t^i + z_t^i) &= (I - EE^\top)\underbrace{(x_t^i + z_t^i)}_{=(p_i, s_t^i+z_t^i)} + E\mathbf{P}(E^\top(p_i, s_t^i + z_t^i))\\
    &=(I - EE^\top)x_t^i + E\mathbf{P}(s_t^i + z_t^i) \\
    &= (p_i, 0) + (0, \mathbf{P}(s_t^i + z_t^i)) = (p_i, \mathbf{P}(z_t^i + s_t^i)) = x_{t+1}^i
\end{align*}
for all $t\in\cT$ and $i\in\cN$. Therefore, the Transformer dynamics preserve the positional information while effectively computing the state of the next vector.

\begin{remark}
    We have several comments on this construction:
    \begin{enumerate}
        \item The positional encodings can be taken as {\it one-hot vectors}: $\mathrm{PE}_N := \{e_i : i\in\cN\} \subset \RR^N$ where $e_i = (\delta_{ij})_{j=1}^N$ and $\delta_{ij}$ is the Kronecker delta. In the literature, the positional encodings are given in terms of many different structures, such as sinodial ones \citep{vaswani2017attention}. For our theoretical results, it suffices that $\mathrm{PE}_N$ satisfies the separation condition $\delta := \min_{i\neq j}\|p_i - p_j\| > 0$. In the example of one-hot vectors, $\delta = \sqrt{2}$.

        \item The formulation can easily be appended by multi-headed self-attention. In such case, we expand the weight space $\UU$ to a one that contains $H\in\NN$ tuples at a time: $U = (U^h)_{h=1}^H$ and $U^h = (W, A, b, Q^h, K^h, V^h)$ with the same dimensions as earlier. Then the Transformer dynamics in Definition~\ref{def:Transformer} take the form
        \[
            x_{t+1}^i = \Pi\left( x_t^i + E \bigg(W_t\sigma(A_tx_t^i + b_t) + \sum_{h=1}^H \sum_{j=1}^N \frac{ \e^{\beta\inner{Q_t^hx_t^i}{K_t^hx_t^j}}}{\sum_{j^\prime=1}^N\e^{\beta\inner{Q_t^hx_t^i}{K_t^hx_t^{j^\prime}}}}V_t^hx_t^j\bigg) \right).
        \]
        All of the structural results in this paper apply to multi-headed case as well. However, we omit them for brevity.
        
        \item In practical applications, $\mathbf{P}$ is a variant of projection such as LayerNorm \citep{ba2016layer}, or RMSNorm \citep{zhang2019root}. Our analysis, however, is based on a general function satisfying the conditions that $\mathbf{P}$ satisfies.
    \end{enumerate}
\end{remark}
}

\subsection{Data set-induced dynamics and information structure} The goal of Transformer training is to minimize some cost functional over the set of weights for a given data set. To this end, suppose that we are given a data set $\cD := \{(\mathbf{s}^k, \mathbf{y}^k): k=1,2,\ldots,K\}\subseteq S^N\times S^N$, for some $K\in\NN$.

For each $k\in\cK := \{1,\ldots,K\}$, the feature vector $\mathbf{s}^k = (s^{1,k},\ldots,s^{N,k})\in S^N$ gives rise to an ensemble of particles $(x_t^{1,k},\ldots,x_t^{N,k})$ that evolves with \eqref{eq:Dynamics} with initial states $s_0^{i,k} = (p_i, s^{i,k})$ for each $i\in\cN$. In our problem, we consider a \emph{centralized information structure}
\begin{equation}
    \cI_t^{i,k} = \cI_t := \{x_s^{i, k}, U_s : i\in\cN, k\in\cK, 0 \leq s \leq t-1\}\cup\{x_t^{i, k} : i\in\cN, k\in\cK\}
\end{equation}
for all $t\in\overline{\cT}\setminus\{0\}$ and all $i\in\cN$. If $t=0$, we have $\cI_0 = \{x_0^{i,k}: i\in\cN, k\in\cK\}$. The equality $\cI_t^{i,k} = \cI_t$ indicates that each particle $x_t^{i,k}$ has access to the same information structure. In optimal control problems, the actions are chosen using {an \emph{admissible policy}, which is a sequence of functions of the form $\gamma_t(\cI_t) = U_t\in\UU$ for each $t\in\cT$. Let $\Gamma$ be the set of all admissible policies.}

\subsection{Cost functional and optimality criterion}
Having specified the dynamics and the information structure, we now complete our formulation of Transformer training as an optimal control problem by defining a cost functional and the associated optimality criterion.

 For any $\underline{\gamma}\in\Gamma$, we write $x_{t,\underline{\gamma}}^{i,k}$ for the particle evolving according to \eqref{eq:Dynamics} where the controls are obtained by $\underline{\gamma}$ and the centralized information structure we defined earlier. For the label vectors, we write $\mathbf{y}^k = (y^{1,k},\ldots,y^{N,k})$ for all $k\in\cK$. We compare the feature and label vectors using a loss function $L:\XX \times S\to\RR_+$. The minimization task is as follows
\begin{align*}
    &\underset{\displaystyle\underline{\gamma}\in\Gamma}{\text{minimize}}\quad\frac{1}{NK}\sum_{i=1}^N\sum_{k=1}^K L(x_{T,\underline{\gamma}}^{i,k}, y^{i,k}) \\
    &\text{subject to: } x_{t,\underline{\gamma}}^{i,k} \text{ evolving according to \eqref{eq:Dynamics}},\\
    &\hspace{5em} x_{0,\underline{\gamma}}^{i,k} = x^{i,k} \text{ for all } i\in\cN, k\in\cK, t\in\cT. \tag{OP}\label{prob:original}
\end{align*}

To this end, we provide another representation of the problem expressed as a finite-horizon terminal-cost-only expected cost functional minimization. Let $P_\cD$ denote the true distribution of the data set $\cD$, that is $P_\cD \in \cP(\XX^N\times S^N)$. For convenience, we write
\begin{equation*}
    \mathbf{L}(\mathbf{x}, \mathbf{y}) := \frac{1}{N}\sum_{i=1}^N L(x^i, y^i)
\end{equation*}
for all $\mathbf{x} := (x^1,\ldots,x^N)$, $\mathbf{y} := (y^1,\ldots,y^N) \in (\RR^d)^N$. We consider a scenario in which the central controller { tries} to find an policy that is optimal for the underlying distribution.
\begin{align*}
    &\hspace{-3em}\underset{\displaystyle\underline{\gamma}\in\Gamma}{\text{minimize}}\quad\EE_{P_\cD}^{\underline{\gamma}}\mathbf{L}(\mathbf{x}_{T}, \mathbf{y}). \tag{ROP}\label{prob:robustness-original}
\end{align*}
Here $\mathbf{x}_T$ denotes the terminal state of the process, whose initial distribution is the first marginal of $P_\cD$, evolving in the trajectory specified by the Transformer dynamics \eqref{eq:Dynamics} and by the policy $\underline{\gamma}\in\Gamma$.

There are several benefits of the presentation of the optimization problem as if the initial data is random:
\begin{itemize}
\item[(i)] This captures the standard setup given in (OP) above, where the data is empirically given, in which case, the distribution is the realized empirical distribution for training.
\item[(ii)] As noted, this allows us to formulate the robustness to initial training data and asymptotic consistency as the data converges weakly or in the Wasserstein sense to a true distribution based on which data is generated.
\item[(iii)] This formulation makes the subtle point of open-loop/closed-loop policy discussion transparent: The policy we obtain is closed loop in the distribution, but open loop in the realization during the implementation of the Transformer for any source once the Transformer is trained.
\end{itemize}

{
\begin{example}
    Here are some practical examples of the learning tasks that can be captured by (OP):
    \begin{enumerate}
        \item In next-token prediction (causal language models), the input $x_0^{i,k}$ is the embedding of $i^{\rm th}$ token and $y^{i,k}$ is the embedding/one-hot of token $i+1$. The most commonly used loss function in this case is the cross-entropy loss.
        \item When denoising an image, we separate an image into chunks and we load this data in a sequential manner to a Transformer. Let $x_0^{i,k}$ be the $i^{\rm th}$ chunk of the noisy version of the $k^{\rm th}$ image in a training sample, consisting of noisy versions of a collection of images. Then $y^{i,k}$ is the $i^{\rm th}$ patch of the noiseless $k^{\rm th}$ image. 
        \item In a teacher-student experiment (such as the one in our paper), $x_0^{i,k}$ is the input of a target function $T$ and $y^{i,k} = T(x_0^{i,k})$. Therefore, in a sense, we are approximating the function via Transformers. The loss function can be chosen any metric that can be computable reasonably fast and that can realize the underlying topology yielding the approximation.
    \end{enumerate}
\end{example}
}

\section{McKean--Vlasov Structure and Measure-Valued MDP}\label{sec:lifting-section}
In this section, we (i) formalize the deterministic McKean--Vlasov representation of the dynamics, (ii) lift the particle-level dynamics to the space of probability measures, (iii) verify the weak Feller property of the resulting transition kernel, (iv) prove the existence of optimal closed-loop policies, and (v) relate these closed-loop policies to initial-data-dependent open-loop policies from the viewpoint of implementation and conformity with the standard training paradigm.

\subsection{Deterministic McKean--Vlasov representation of the particles.} We continue by defining the McKean--Vlasov map governing the Transformer dynamics given in~\eqref{eq:McKean--Vlasov}.

{
\begin{definition}\label{def:McKeanVlasov}
    The function $f:\XX\times\UU\times\cP(\XX)\to\XX$ is defined as
    \begin{equation}\label{eq:McKean--Vlasov}
        f(x, U, \mu) := \Pi\left(x + E\bigg(W\sigma(Ax + b) + \dfrac{\int_\XX \mu(\de z)\,\e^{\beta\inner{Qx}{Kz}}\,Vz}{\int_\XX \mu(\de \tilde{z})\,\e^{\beta\inner{Qx}{K\tilde{z}}}}\bigg)\right),
    \end{equation}
    for all $x\in\XX$, $U = (W, A, b, Q, K, V)\in\UU$ and $\mu\in\cP(\XX)$.
\end{definition}
To make this concrete, take $x = x_t^i$, $U_t\in\UU$ as in \eqref{eq:Dynamics}, and $\mu_t^k = \frac{1}{N}\sum_{i=1}^N\delta_{x_t^i}$ for all $i\in\cN$, $k\in\cK$, and $t\in\cT$. Then
\begin{align*}
    f(x_t^i, U_t, \mu_t^k) &= \Pi\left(x_t^i + W_t\sigma(A_tx_t^i + b_t) + \dfrac{\int_\XX \mu_t(\de z)\,\e^{\beta\inner{Q_tx_t^i}{K_tz}}\,V_tz}{\int_\XX \mu_t(\de \tilde{z})\,\e^{\beta\inner{Q_tx_t^i}{K_t\tilde{z}}}}\right)\\
    &= \Pi\left( x_t^i + E \bigg(W_t\sigma(A_tx_t^i + b_t) + \sum_{j=1}^N \frac{ \e^{\beta\inner{Q_tx_t^i}{K_tx_t^j}}}{\sum_{j^\prime=1}^N\e^{\beta\inner{Q_tx_t^i}{K_tx_t^{j^\prime}}}}V_tx_t^j\bigg) \right) = x_{t+1}^i.
\end{align*}
}
Note that the flow of any particle does not solely depend on its previous realizations, but it depends explicitly on the empirical { measure} of the ensemble $\mu_t^k$.

\subsection{The lifted system and its continuity properties.} We continue with precisely explaining the lifting procedure. Then, we prove that the underlying dynamics are continuous under the following assumption:

\begin{assume}\label{assume:compactness}
     The action space $\UU$ is compact and $\sigma$ is $1$-Lipschitz.
\end{assume}

We endow the space of probability measures with the weak\textsuperscript{*}-topology, that is a sequence of probability measures $(\mu_n)_{n\in\NN}$ \emph{$\text{weak}^*$-convergent} to $\mu\in\cP(\XX)$ if, for any continuous and bounded $h:\XX\to\RR$, we have
\begin{equation*}
    \int_\XX\mu_n(\de x)\,h(x) \xrightarrow{n\to\infty} \int_\XX\mu(\de x)\,h(x).
\end{equation*}
Since $\XX$ is compact, $\cP(\XX)$ with the $\text{weak}^*$-topology is compact. Moreover, we have many other choices for a suitable metric. In particular, we consider a modified version of $p$-Wasserstein distance, $p\in[1,+\infty)$, defined by
\begin{equation}\label{eq:p-wasserstein}
    W_{p,\lambda}(\mu,\nu) := \inf_{\pi\in\Pi(\mu,\nu)}\left(\int_{\XX\times\XX}\pi(\de x, \de y) \|x-y\|_{p,\lambda}^p \right)^{\frac{1}{p}}\quad \mu,\nu\in\cP(\XX),
\end{equation}
where $\|x-y\|_{p,\lambda}^p = \|s_x-s_y\|_p^p + \lambda|p_x-p_y|^p$, $x = (p_x, s_x), y = (p_y, s_y)\in\XX$, $\lambda > 0$, and
\begin{equation*}
    \Pi(\mu,\nu) := \{\pi\in\cP(\XX\times\XX) : \pi(\cdot\times\XX) = \mu(\cdot) \text{ and } \pi(\XX\times\cdot) = \nu(\cdot)\}.
\end{equation*}
Here $\Pi(\mu,\nu)$ is the set of couplings of $\mu$ and $\nu$, i.e., all joint probability measures on $\XX\times\XX$ such that the respective marginals of each $\pi\in\Pi(\mu,\nu)$ are $\mu$ and $\nu$.

{

\begin{definition}\label{def:MeasureFlow}
    We define the map $\Phi:\cP(\XX)\times\UU\to\cP(\XX)$ for all $\mu\in\cP(\XX)$ and $U\in\UU$ by calculating the following push-forward measure
    \begin{equation}\label{eq:measure-evolution}
        \Phi(\mu, U) := \mu\circ T_{\mu, U}^{-1}
    \end{equation}
    where $T_{\mu, U}:\XX\to\XX$ is given by $T_{\mu, U}(x) := f(x,U,\mu)$ for all $x\in\XX$.
\end{definition}
}
This can be viewed as the discrete-time analogue of the continuity equation in \citep{geshkovski2025mathematical}. Consequently, we have a concrete functional representation of the measure flow in this particular ensemble controlled McKean--Vlasov system. It follows that, by a change of variables argument, for any measurable $g:\XX\to\RR$, $k\in\cK$ and $t\in\cT$, we have
{
\begin{align*}
    \int_\XX \Phi(\mu_t^{k}, U_t)(\de x)\, g(x) &= \int_\XX \big(\mu_t^k\circ T_{\mu_t^k, U_t}^{-1}\big)(\de x)\, g(x) \\
    &= \int_\XX\frac{1}{N}\sum_{i=1}^N \delta_{x_t^{i,k}}(T_{\mu_t^k, U_t}^{-1}(\de x)) g(x) \\
    &= \frac{1}{N}\sum_{i=1}^N g(T_{\mu_t^k, U_t}(x_t^{i,k})) \\
    &= \frac{1}{N}\sum_{i=1}^N g(f(x_t^{i,k}, U_t, \mu_t^k)) \\
    &= \frac{1}{N}\sum_{i=1}^N g(x_{t+1}^{i,k}) = \int_\XX\mu_{t+1}^k(\de x)\, g(x).
\end{align*}}
Thus, $\Phi(\mu_t^k, U_t) = \mu_{t+1}^k$ for all $t\in\cT$ and $k\in\cK$. This indicates that each empirical measure evolves with a Markovian structure. Since our objective is to minimize a cost criterion over all training data, and consequently, all empirical measures, we consider a multi-variable generalization of $\Phi$ as $\boldsymbol{\Phi}:\cP(\XX)^K\times\UU\to\cP(\XX)^K$ is given, for all $(\mu^1,\ldots,\mu^K)\in\cP(\XX)^K$ and $U\in\UU$, by
\begin{equation}\label{eq:ensemble-dynamics}
    \boldsymbol{\Phi}(\mu^1,\ldots,\mu^K; U) = (\Phi(\mu^1, U), \ldots, \Phi(\mu^K, U)).
\end{equation}
This generalization allows us to control the minimization problem over the data set, a similar approach can be found in \citep{E2019}. On the other hand, we can now see the ensemble process $((\mu_t^1,\ldots,\mu_t^K))_{t\in\overline{\cT}}$, initialized from the data set $\cD$ and evolving with $\boldsymbol{\Phi}$, as a deterministic and measure-valued Markov decision process. This ensemble view allows us to use a cost criterion to which all points in the data set contribute. To continue, we now establish the continuity of $\Phi$. To do this, since we construct $\Phi$ using $f$, we first need the continuity of $f$ under Assumption \ref{assume:compactness}.
{
\begin{proposition}\label{prop:Continuity}
    Under Assumption \ref{assume:compactness}, the following assertions are true:
    \begin{enumerate}
        \item The function $f$ in \eqref{eq:McKean--Vlasov} is jointly continuous on $\XX\times\UU\times\cP(\XX)$
        \item $\Phi$ defined in \eqref{eq:measure-evolution} is jointly continuous on $\cP(\XX)\times\UU$.
    \end{enumerate}
    In all assertions, $\cP(\XX)$ is endowed with $\text{weak}^*$-topology.
\end{proposition}
\begin{proof}
    See Appendix~\ref{proof:Continuity}.
\end{proof}
}

{\begin{definition}\label{def:Kernel}
     We define the transition kernel $\eta$ of $(\mu_t^1,\ldots,\mu_t^K)_{t\in\overline{\cT}}$ by
    \begin{align}\label{eq:transition-kernel}
        \eta(B \mid \mu_t^1,\ldots,\mu_t^K; U_t) &:= \PP((\mu_{t+1}^1,\ldots,\mu_{t+1}^K)\in B \mid \mu_t^1,\ldots,\mu_t^K; U_t)\\
        &= \PP\{\boldsymbol{\Phi}(\mu_t^1,\ldots,\mu_t^K; U_t) \in B\} = \delta_{\boldsymbol{\Phi}(\mu_t^1,\ldots,\mu_t^K; U_t)}(B),
    \end{align}
    where $B\in\cB(\cP(\XX)^K)$.
\end{definition}

In Definition~\ref{def:Kernel}, we endow $\cP(\XX)^K$ with the product $\text{weak}^*$-topology and obtain the Borel $\sigma$-field structure.}

Note that the continuity of $\Phi$ implies the continuity of $\boldsymbol{\Phi}$, which is critical since we want the transition kernel $\eta$ to be \emph{weak Feller}, that is, for any continuous and bounded function $g:\cP(\XX)^K\to\RR$, we have
\begin{equation}\label{eq:weak-feller}
    \cP(\XX)^K\times\UU\ni(\mu^{1:K}, U)\mapsto\int_{\cP(\XX)^K}\eta(\de\mu_{t+1}^{1:K} \mid \mu^{1:K}, U)\,g(\mu_{t+1}^{1:K})
\end{equation}
is a continuous function.

\begin{corollary}\label{cor:transition-kernel}
    The transition kernel $\eta$ defined in \eqref{eq:transition-kernel} is weak Feller under Assumption \ref{assume:compactness}.
\end{corollary}
\begin{proof}
    See Appendix~\ref{proof:transition-kernel}.
\end{proof}

\subsection{Existence of optimal policies}\label{sub:existence-of-optimality}
We now employ the framework of Markov decision processes. First, we observe that the ensemble of empirical measures $\left((\mu_t^1,\ldots,\mu_t^K)\right)_{t\in\overline{\cT}}$ is a Markov decision process defined on $\cP(\XX)^K$ with the action space $\UU$ where
\begin{itemize}
    \item $\cP(\XX)^K$ is endowed with the product $\text{weak}^*$-topology.
    \item for any $\boldsymbol{\mu}\in\cP(\XX)^K$ and $U\in\UU$, the transition kernel is given by
    \begin{equation}
    \eta(\,\cdot\mid \boldsymbol{\mu}, U) = \delta_{\boldsymbol{\Phi}(\boldsymbol{\mu}, U)}(\cdot).
    \end{equation}
    \item Our aim is to minimize the  following cost functional:
    \begin{equation}\label{eq:cost-functional}
    V(\cD; \underline{\gamma}) =\frac{1}{K}\sum_{k=1}^K W_{2,\lambda}(\mu_T^k, \nu^k)^2
    \end{equation}
    over the set of all strategies $\Gamma$ for Transformers, where $\mu_T^k$ is initialised from the $k^{\text{th}}$-sample of the data set $\cD$ according to the flow $\Phi$ using the policy $\underline{\gamma}$.
\end{itemize}

\begin{theorem}\label{thm:equivalence}
    If \ka{$\lambda > \operatorname{diam}(S)^2/\delta^2$}, then the lifted problem and (OP) are equivalent.
\end{theorem}
\begin{proof}
    See Appendix~\ref{proof:equivalence}.
\end{proof}

We impose the following assumption in order to have a position-sensitive cost structure.

\begin{assume}\label{assume:lambda}
    The choice of $\lambda$ is done in a way that \ka{$\lambda > \operatorname{diam}(S)^2/\delta^2$} holds.
\end{assume}

Note that, we have written $V$ in a way it depends on the initial data. Since the problem is fixed for a given data set, we omit it and simply write $V(\underline{\gamma})$ for the same cost. In this expectation, we start with an ensemble of initial empirical measures, i.e. $\mu_0^k := \frac{1}{N}\sum_{i=1}^N\delta_{X_0^{i,k}}$ for all $k\in\cK$, and we transport the ensemble via $\boldsymbol{\Phi}$

\begin{figure}[H]
\centering
\resizebox{0.95\linewidth}{!}{
\begin{tikzpicture}[
  node distance=7mm,
  box/.style={rounded corners, align=center, inner sep=4pt, font=\small},
  arr/.style={-Latex, thick}
]
\node[box] (t0) {$\big(\mu_0^1,\ldots,\mu_0^K\big)$};

\node[box, right=14mm of t0] (t1) {$\big(\mu_{1}^1,\ldots,\mu_{1}^K\big)$};

\node[box, right=14mm of t1] (dots) {$\cdots$};

\node[box, right=14mm of dots] (t) {$\big(\mu_t^1,\ldots,\mu_t^K\big)$};

\node[box, right=14mm of t] (dots2) {$\cdots$};

\node[box, right=14mm of dots2] (T) {$\big(\mu_{T}^1,\ldots,\mu_{T}^K\big)$};

\draw[arr] (t0) -- node[above, font=\small] {$\boldsymbol{\Phi}(\cdot, U_0)$} (t1);
\draw[arr] (t1) -- node[above, font=\small] {$\boldsymbol{\Phi}(\cdot, U_1)$} (dots);
\draw[arr] (dots) -- node[above, font=\small] {$\boldsymbol{\Phi}(\cdot, U_{t-1})$} (t);
\draw[arr] (t) -- node[above, font=\small] {$\boldsymbol{\Phi}(\cdot, U_{t})$} (dots2);
\draw[arr] (dots2) -- node[above, font=\small] {$\boldsymbol{\Phi}(\cdot, U_{T-1})$} (T);
\end{tikzpicture}
}

\caption{Flow of the ensemble of empirical measures on particle via the map $\boldsymbol{\Phi}$.}
\label{fig:phi-bold-map}
\end{figure}
under $\underline{\gamma}\in\Gamma$.

\begin{theorem}\label{thm:Existence}
    Let Assumptions \ref{assume:compactness} and \ref{assume:lambda} hold. Let us define the dynamic programming equations
    \begin{equation}
        \begin{split}
            &\fC_T(\mu^1,\ldots,\mu^K) :=  \frac{1}{K}\sum_{k=1}^K W_{2,\lambda}(\mu^k,\nu^k)^2 \\
            &\fC_t(\mu^1,\ldots,\mu^K) := \inf_{U\in\UU}\int_{\cP(\XX)^K}\eta(\de\mu_{t+1}^{1:K}\mid \mu^{1:K}; U)\,\fC_{t+1}(\mu_{t+1}^{1:K})
        \end{split}
    \end{equation}
    for all $(\mu^1,\ldots,\mu^K)\in\cP(\XX)^K$ and $t\in\cT$. Then there exists $\underline{\gamma}^*\in\Gamma$ such that
    \begin{equation}
        V(\underline{\gamma}^*) = \inf_{\underline{\gamma}\in\Gamma}V(\underline{\gamma})
    \end{equation}
    and $\underline{\gamma}^* = (\gamma_t^*)_{t\in\cT}$ is of class deterministic Markov in the sense that $\gamma_t^*(\mu_t^1,\ldots\mu_t^K) = U_t^*\in\UU$.
\end{theorem}
\begin{proof}
    See Appendix~\ref{proof:Existence}.
\end{proof}

Note that, we could have written the dynamic programming equations in an equivalent way as follows: For any $t\in\cT$ and $\mu^1,\ldots,\mu^K\in\cP(\XX)$, we have
\begin{align*}
    \fC_t(\mu^1,\ldots,\mu^K) &= \inf_{U\in\UU}\int_{\cP(\XX)^K}\eta(\de\mu_{t+1}^{1:K}\mid \mu^{1:K}; U)\,\fC_{t+1}(\mu_{t+1}^{1:K}) \\
    &= \inf_{U\in\UU}\fC_{t+1}(\boldsymbol{\Phi}(\mu^1,\ldots,\mu^K; U))
\end{align*}
and in particular
\begin{align*}
    \fC_{T-1}(\mu^1,\ldots,\mu^K) &= \inf_{U\in\UU}\fC_T(\boldsymbol{\Phi}(\mu^1,\ldots,\mu^K; U))\\
    &= \inf_{U\in\UU}\sum_{k=1}^K W_{2,\lambda}(\Phi(\mu^k, U),\nu^k)^2.
\end{align*}
By induction, if we continue, we eventually have
\begin{equation*}
    V^* = \fC_0(\mu_0^1,\ldots,\mu_0^K) = \inf_{U_0\in\UU}\cdots\inf_{U_{T-1}\in\UU}\frac{1}{K}\sum_{k=1}^K W_{2,\lambda}(\Phi(\cdots\Phi(\mu_0^k, U_0),U_1,\cdots,U_{T-1}), \nu^k)^2.
\end{equation*}
With dynamic programming, we first minimize starting from inside to outside, meaning that we first minimize starting from the terminal cost. At each time-step, we work with only optimization of one term and get optimal control action for each state. If we select (measurably) one optimal control action for each state, we obtain so-called a deterministic Markovian policy; a function deterministically assigns a control according to only the information available at the same time-step. This procedure is more efficient than checking each action individually even if we have a finite set of admissible actions. However, we have a continuum of actions, which is not quite easy to work with currently. In the next section, we reduce this set to a finite one at the cost of having near-optimality. However, by doing this, we depend on the compactness of state and action spaces instead of relying on the convexity and smoothness of the problem, which is not the case in most applications.

\subsection{Open-loop policy design via closed-loop policy design for the lifted problem.}
Let $\underline{\gamma}^*\in\Gamma$ be a deterministic Markov strategy obtained from the dynamic programming equations. Although, with infinite state and action spaces, calculation of $\underline{\gamma}^*$ is very difficult, it corresponds to the training phase in classical approach. In the next section, we provide an approach to train Transformers with dynamic programming using quantization.

We now discuss why open-loop policies constitute the appropriate class of policies in the context of Transformers. An \emph{open-loop} policy refers to a collection of control actions $(U_t)_{t\in\cT}$ for all $t\in\cT$. In other words, an open-loop policy is a finite sequence in the action space. A \emph{closed-loop} policy, also called \emph{feedback} policy, is given by a sequence of measurable maps $(\gamma_t)_{t\in\cT}$ of the form $\gamma_t(\mu_t^1,\ldots,\mu_t^K) = U_t\in\UU$ for all $t\in\cT$. Basically, a closed-loop policy yields an action based on the current state.

In the operational context of Transformers, one would like to fix the architecture after the training process is finished. In the case of closed-loop policies, the controls are calculated at each layer $t\in\cT$, which is not computationally efficient, and also, not compatible with classical paradigm.

We establish and build on the following equivalence for classes policies:
\begin{center}
 A {\it closed-loop policy of the lifted problem} is equivalent to an {\it initial-distribution dependent open-loop policy} (owing to the deterministic and ensemble nature of the lifted problem), which is, in turn, equivalent to a {\it realized-input-independent open loop policy}, and thus compatible with Transformer design (see Figure \ref{EquivalenceControlPol}).
\end{center}

Again let $\underline{\gamma}^*\in\Gamma$ be an optimal closed-loop for the lifted model obtained via Theorem \ref{thm:Existence}. For $t=1$, we may write
\begin{equation*}
    U_1^* = \gamma_1^*(\mu_1^1,\ldots,\mu_1^K) = \gamma_1^*(\boldsymbol{\Phi}(\mu_0^1,\ldots,\mu_0^1; \gamma_0^*(\mu_0^1,\ldots, \mu_0^K))).
\end{equation*}
In other words, once we learn $\gamma_0^*$ and $\gamma_1^*$, we can calculate optimal action $U_1^*$ using only the initial data. By induction, it is possible to find functions $g_0^*,\ldots,g_T^*:\cP(\XX)^K\to\UU$ by, for any $t\in\cT$, we expand
\begin{align*}
    U_t^* &= \gamma_t^*(\mu_t^1,\ldots,\mu_t^K) \\
    &=\gamma_t^*(\boldsymbol{\Phi}(\mu_{t-1}^1,\ldots,\mu_{t-1}^K;\gamma_{t-1}^*(\mu_{t-1}^1,\ldots,\mu_{t-1}^K)))\\
    &\hspace{1cm}\vdots\\
    &=\gamma_t^*(\boldsymbol{\Phi}(\cdots\boldsymbol{\Phi}(\mu_0^1,\ldots, \mu_0^K;\gamma_0^*(\mu_0^1,\ldots, \mu_0^K), \gamma_1^*(\boldsymbol{\Phi}(\mu_0^1,\ldots, \mu_0^K;\gamma_0^*(\mu_0^1,\ldots, \mu_0^K))\cdots))))\\
    &=: g_t^*(\mu_0^1,\ldots,\mu_0^K).
\end{align*}

\begin{figure}[H]
\centering
\begin{tikzpicture}[
    scale=1,
    transform shape,
  arr/.style={-Latex, thick},
  lab/.style={font=\small, align=center},
  node distance=10mm and 16mm
]

\node[lab] (ol_title) {\textbf{Open-loop control}};
\node[lab, below=of ol_title] (ol_mu) {$\boldsymbol{\mu}_{t}$};
\node[lab, right=of ol_mu] (ol_phi) {$\Phi$};
\node[lab, below=of ol_phi] (ol_u) {$U_t^*$\\(fixed by $g_t^*(\boldsymbol{\mu}_{0})$)};
\node[lab, right=of ol_phi] (ol_mu_next) {$\boldsymbol{\mu}_{t+1}$};

\draw[arr] (ol_mu) -- (ol_phi);
\draw[arr] (ol_phi) -- (ol_mu_next);
\draw[arr] (ol_u) -- (ol_phi);

\node[lab, right=35mm of ol_title] (cl_title) {\textbf{Closed-loop control}};
\node[lab, below=of cl_title] (cl_mu) {$\boldsymbol{\mu}_{t}$};
\node[lab, right=of cl_mu] (cl_phi) {$\Phi$};
\node[lab, below=of cl_phi] (cl_u) {$U_t^*=\gamma_t^*(\boldsymbol{\mu}_{t})$};
\node[lab, right=of cl_phi] (cl_mu_next) {$\boldsymbol{\mu}_{t+1}$};

\draw[arr] (cl_mu) -- (cl_phi);
\draw[arr] (cl_phi) -- (cl_mu_next);
\draw[arr] (cl_u) -- (cl_phi);
\draw[arr, dashed] (cl_mu) -- (cl_u);

\end{tikzpicture}
\caption{Comparison of information structures for open-loop and closed-loop controls. In the open-loop case, $U_t^*$'s are fixed before and do not depend on the current measure unlike the closed-loop case.}
\label{fig:open-vs-closed}
\end{figure}

In short, by recursively expanding the expression $\gamma_t^*(\mu_t^1,\ldots,\mu_t^K)$ using the flow $\boldsymbol{\Phi}$ and preceding $\gamma_0^*,\ldots,\gamma_{t-1}^*$, we express $U_t^*$ as a function $F_t$ of $\mu_0^1,\ldots, \mu_0^K$ and $\gamma_0^*,\ldots,\gamma_{t-1}^*$. Then, we simply call it $g_t^*$. This collection of maps $(g_t^*)_{t\in\cT}$ is classified as an \emph{initial condition-dependent open-loop policy}. Since $\Phi$ and $\gamma_t^*$'s are deterministic, so is each $g_t^*$. Moreover, since $\Phi$ is known \emph{a priori} and $(\gamma_t^*)_{t\in\cT}$ is learned via dynamic programming, we know how to calculate each $g_t^*$. Therefore, we do not need to calculate actions at each time-step, which allows us to fix the architecture once we finish the training phase, as intended. Then, whenever new data is fed into the trained Transformer, the corresponding output will be given with the learned controls, or weights as it is done with the classical methods.

\section{Triply Quantized Training Scheme for Transformers}\label{sec:quantization}
We have so far developed a mathematical framework and control theoretic formulation on optimal design of Transformer architecture. In this section, building on this formulation, we present a rigorous approximation method which is provably near optimal, thus facilitating a numerical implementation for Transformer design with performance guarantees.

{
We exploit the total boundedness of $\XX$. More concretely, for any $n\in\NN$, it is possible to find a set $\{x^{(n), j}: j\in J(n)\} \subseteq \XX$ such that
\begin{equation*}
    \min_{j\in J(n)} \|x - x^{(n),j}\| < \frac{1}{n}
\end{equation*}
for all $x\in\XX$.} Note that compactness, thereby total boundedness, of $\XX$ ensures that the indexing set $J(n)$ is finite. It is possible to include the initial information into the quantizer set. However, the analysis provided in this section is asymptotically consistent even if we do not include them. So, write
\begin{equation}
    \XX_n := \{x^{(n), j}: j\in J(n)\}
\end{equation}
for all $n\in\NN$. The set $\XX_n$ is finite and its cardinality is $|J(n)|$.

\begin{definition}\label{def:Quantizer}
    We define the \emph{nearest neighbor quantizer} $Q_n: \XX \to \XX_n$ as follows
    {
    \begin{equation}
        Q_n(x) := \operatorname{argmin}_{\hat{x}\in\XX_n}\|x-\hat{x}\|, \quad x\in\XX,
    \end{equation}}
    \noindent where the ties are broken in a way to protect measurability.
\end{definition}

Moreover, to have a feasible learning algorithm, we need to quantize the action space $\UU$. Recall that we endow $\UU$ with the product Frobenius norm, that is $\|U\|_\UU^2 = \Tr(WW^T) + \Tr(AA^T) + \Tr(bb^T) + \Tr(QQ^T) + \Tr(KK^T) + \Tr(VV^T)$ for all $U=(W,A,b,Q,K,V)\in\UU$. By Assumption \ref{assume:compactness}, we also know that $\UU$ is compact, as a result, for any $m\in\NN$, there is a finite set $\UU_m := \{U^1, \ldots, U^{(M(m))}\} \subseteq \UU$ such that
\begin{equation}
    \min_{j=1,\ldots,M(m)}\|U^{(m),j}-U\|_\UU \leq \frac{1}{m}
\end{equation}
for all $U\in\UU$.

Although the problem is reduced from an infinite-dimensional problem to a finite-dimensional one, it remains computationally expensive. If we were to consider the model where only the states are quantized, we would have a product of probability simplices $\cP(\XX_n)^K$ in $\RR^{|\XX_n|}$. In particular, we see each member of this product as follows
\begin{equation}
\cP(\XX_n) = \{p = (p_a)_{a=1}^{|\XX_n|}\in\RR^{|\XX_n|} : p_a \geq 0 \text{ and } \sum_{a=1}^{|\XX_n|} p_a = 1\}.
\end{equation}

Since the state space is still compact, it seems simple to quantize using a similar method. However, one major problem is protecting the initial information and correctly pass it through the layers of Transformer. In this regard, we use algorithm of \cite{reznik2011algorithm}. To track the information passing through the layers, we use the following enumeration of $\XX_n$
\begin{equation}
    \{1,2,...,|\XX_n|\}\ni a \mapsto \begin{cases}
        \displaystyle\left(\bigg\lfloor\frac{a}{|\XX_n|} \bigg\rfloor + 1, a \text{ mod } |\XX_n| + 1\right)& \text{ if } |\mathrm{S_n}| \text{ does not divide } a\\
        \vspace{5pt}\\
        \displaystyle\left(\frac{a}{|\XX_n|}, |\XX_n|\right) & \text{otherwise}
    \end{cases}
\end{equation}
The requirement for going back to a non-quantized measure from a quantized one is because of the definition of $\tilde{f}$, where we use the original $f$, which takes atoms in $\XX_n \subset \XX$. Here the first component is associated with the corresponding positional value and the second component is the index of a particle state in $\XX_n$ with its predetermined enumeration, that is, we associate each $(i,j)$ to $(p_i, x^{(n),j})$ in $\XX_n$. So, we basically write $(X_a)_{a=1}^{|\XX_n|}$ for $\XX_n$ and we know exactly what $X_a$ is for a given $a\in\{1,2,...,|\XX_n|\}$. Then, as in the paper of \cite{reznik2011algorithm}, we define the \emph{reconstruction points} of $\cP(\XX_n)$ as the set
\begin{equation}
    \cP^{(\ell)}(\XX_n) := \bigg\{(p_a)_{a=1}^{|\XX_n|}: p_a = \frac{r_a}{\ell}, r_a \in \NN_0, \sum_{a=1}^{|\XX_n|}r_a = \ell\bigg\}.
\end{equation}
This is a finite set with $\binom{\ell + |\XX_n| - 1}{|\XX_n| - 1}$ many points. Let $R_\ell: \cP(\XX_n)\to\cP^{(\ell)}(\XX_n)$ be the quantizer given by \cite[Algorithm 1]{reznik2011algorithm}. Then $R_\ell$ maps each $\hat{\mu}\in\cP(\XX_n)$ to the closest element in $\cP^{(\ell)}(\XX_n)$ in the sense of $W_{2,\lambda}$ metric.

\begin{proposition}\label{prop:resnik-quantizer-error}
    For fixed $n\in\NN$ and $\hat{\mu}\in\cP(\XX_n)$, then $W_{2,\lambda}(\hat{\mu}, R_\ell(\hat{\mu}))\to 0$ as $\ell\to+\infty$.
\end{proposition}
\begin{proof}
    See Appendix~\ref{proof:resnik-quantizer-error}.
\end{proof}

Note that the bound in the proof of Proposition \ref{prop:resnik-quantizer-error} is not stable with respect to $n$, i.e. cardinality of the state quantizer space $|\XX_n|$ increase faster than $\cO(n^2)$. To fix this issue, we impose the following assumption.
\begin{assume}\label{assume:growth-assumptions}
    The state quantization is done with balls with radius $\frac{1}{n}$ and the measure quantization index $\ell$ is a function of $n$ of order $\cO(n^{d+1})$.
\end{assume}
Covering a compact set in $\RR^d$ with $\frac{1}{n}$-nets is, in the worst case, polynomial $\cO(n^d)$. Therefore, with the growth assumption on $\ell$, we guarantee that the state quantization level will not interfere with measure quantization in an undesired way.

Returning to the original problem, consider the data set $\cD = \{(\mathbf{s}^k, \bm{y}^k): k\in\cK\}$. We define the quantized initial empirical measures
\begin{equation*}
    \hat{\mu}_0^k := R_\ell\bigg(\frac{1}{N}\sum_{i=1}^N\delta_{Q_n(x_0^{i,k})}\bigg), \quad k\in\cK
\end{equation*}
where $x_0^{i,k} = (p_i, s^{i,k})$ for all $i\in\cN$ and $k\in\cK$. To capture the dynamics, define the map $f_n:\XX_n\times\UU\times\cP(\XX_n) \to \XX_n$ by $f_n(\hat{x}, U, \hat{\mu}) := Q_n(f(\hat{x}, U, \hat{\mu}))$ for all $\hat{x}\in\XX_n$, $U\in\UU$ and $\hat{\mu}\in\cP(\XX_n)$. Also, we define the flow map $\Phi^{(n)}:\cP(\XX_n)\times\UU\to\cP(\XX_n)$ by
\begin{equation}
    \Phi^{(n)}(\hat{\mu}, U) := \hat{\mu}\circ f_n^{-1}(\cdot, U, \hat{\mu})
\end{equation}
for all $\hat{\mu}\in\cP(\XX_n)$ and $U\in\UU$. Although, $f_n$ and $\Phi^{(n)}$ quantize the dynamics in \eqref{eq:ensemble-dynamics}, they are not of McKean--Vlasov type anymore, only particle and measure level approximations of it, respectively. However, it is possible to lift this problem to an MDP as well, which is plainly an approximation to its counterpart. As an intermediate step, we define the transition kernel as
\begin{equation}
    \eta^{(n)}(\,\cdot\mid\hat{\mu}^1,\ldots,\hat{\mu}^K; U) := \delta_{\boldsymbol{\Phi}^{(n)}(\hat{\mu}^1,\ldots,\hat{\mu}^K; U)}(\cdot)
\end{equation}
for all $\hat{\mu}^1,\ldots,\hat{\mu}^K\in\cP(\XX)$ and $U\in\UU$. Here, $\boldsymbol{\Phi}^{(n)}$ is defined in the same way as in the original problem, in other words,
\begin{equation*}
    \boldsymbol{\Phi}^{(n)}(\hat{\mu}^1,\ldots,\hat{\mu}^K; U) := (\Phi^{(n)}(\hat{\mu}^1, U),\ldots,\Phi^{(n)}(\hat{\mu}^K,U)).
\end{equation*}

Now, we define a new MDP, denoted $\mathsf{MDP}^{(\ell,n,m)}$, whose state space is $\cP^{(\ell)}(\XX_n)^K$ and its action space is $\UU_m$. We define the corresponding dynamics via
\begin{equation}
    \Phi^{(\ell,n,m)}(\hat{\mu}, U^{(m)}) := R_\ell(\Phi^{(n)}(\hat{\mu}, U^{(m)}))
\end{equation}
for all $\hat{\mu}\in\cP^{(\ell)}(\XX_n)$ and $U^{(m)}\in\UU_m$. Like we have done before, its multivariable version
\begin{equation}
    \boldsymbol{\Phi}^{(\ell,n,m)}(\hat{\mu}^1,\ldots,\hat{\mu}^K; U^{(m)})= (\Phi^{(\ell,n,m)}(\hat{\mu}^1, U^{(m)}),\ldots,\Phi^{(\ell,n,m)}(\hat{\mu}^K, U^{(m)}))
\end{equation}
allows us to define the transition kernel
\begin{equation}\label{eq:doubly-quantized-transition}
    \eta^{(\ell,n,m)}(\cdot | \hat{\mu}^1,\ldots,\hat{\mu}^K; U^{(m)}) := \delta_{\boldsymbol{\Phi}^{(\ell,n,m)}(\hat{\mu}^1,\ldots,\hat{\mu}^K; U^{(m)})}(\cdot).
\end{equation}
We end the specification of $\mathsf{MDP}^{(\ell,n,m)}$ by indicating that we use a similar cost structure given by an $|\XX_n|\times|\XX_n|$-matrix,
\begin{equation*}
    \mathbf{C}_{a,b} := \|x_a - x_b\|_{2,\lambda}^2
\end{equation*}
for all $x_a, x_b\in\XX_n$ and $a,b\in\{1,\ldots,|\XX_n|\}$, which induces the measure-level cost
\begin{equation*}
    \hat{W}_{2,\lambda}(\hat{\mu},\hat{\nu}) := \min_{\bm{\hat{\pi}}\in\hat{\Pi}(\hat{\mu},\hat{\nu})} \sum_{a=1}^{|\XX_n|}\sum_{b=1}^{|\XX_n|} \mathbf{C}_{a,b}\bm{\hat{\pi}}_{a,b}
\end{equation*}
where $\hat{\Pi}(\hat{\mu},\hat{\nu}) := \{\hat{\pi}\in\RR^{|\XX_n|\times|\XX_n|} : \bm{\hat{\pi}}_{a,b} \geq 0,\  \sum_{a=1}^{|\XX_n|}\bm{\hat{\pi}}_{a,b} = \hat{\nu}_b,\  \sum_{b=1}^{|\XX_n|}\bm{\hat{\pi}}_{a,b} = \hat{\mu}_a\}$.

For this MDP, invoking Blackwell's theorem (see \citep{blackwell1964memoryless}), it is sufficient to use the Markovian information structure, which is given by
\begin{equation*}
    \cI_{t}^{(\ell,n,m)} := (\hat{\mu}_{t}^{k})_{k\in\cK}
\end{equation*}
for all $t\in\overline{\cT}$. We denote the collection of deterministic Markovian (feed-back) strategies by $\Gamma_{\mathrm{DS}}^{(\ell,n,m)}$ whose elements are finite sequences $(\hat{\gamma}_{t}^{(\ell,n,m)})_{t\in\cT}$ of measurable functions such that
\begin{equation*}
    \hat{\gamma}_{t}^{(\ell,n,m)}(\cI_{t}^{(\ell,n,m)}) = U_t^{(m)}\in\UU_m
\end{equation*}
for each $t\in\cT$. Note that we denote the terminal cost functional for $\mathsf{MDP}^{(\ell,n,m)}$ as follows
\begin{equation}
    \hat{V}(\underline{\hat{\gamma}}^{(\ell,n,m)}) := \frac{1}{K}\sum_{k=1}^K \hat{W}_{2,\lambda}(\hat{\mu}_T^{\underline{\hat{\gamma}}^{(\ell,n,m)},k}, \hat{\nu}^k)^2
\end{equation}
for all $\underline{\hat{\gamma}}^{(\ell,n,m)}\in\Gamma_{\mathrm{DS}}^{(\ell,n,m)}$. Here we have $\hat{\nu}^k := R_\ell(\nu^k)$ and $\hat{\mu}_T^{\underline{\hat{\gamma}}^{(\ell,n,m)},k}$ is the quantized terminal measure-state initialized from $\hat{\mu}_0^k$ evolving with $\Phi^{(\ell,n,m)}$ under $\underline{\hat{\gamma}}$ for each $k\in\cK$.
\begin{remark}\label{remark:mdp_ell_n_m}
    Since both $\cP^{(\ell)}(\XX_n)$ and $\UU_m$ are finite, the appropriate topology on $\cP^{(\ell)}(\XX_n)^K\times\UU_m$ is discrete topology. That makes $\boldsymbol{\Phi}^{(\ell,n,m)}$ jointly continuous on its domain. Recalling the proof of Corollary \ref{cor:transition-kernel}, we know that $\eta^{(\ell,n,m)}$ is weak Feller. Moreover, finiteness of $\cP^{(\ell)}(\XX_n)^K\times\UU_m$ automatically ensures continuity and boundedness of the cost structure.
\end{remark}

\begin{theorem}\label{thm:triply-quantized-optimal-policies}
    Let Assumptions \ref{assume:compactness} and \ref{assume:lambda} hold. We pose the dynamic programming equations for $\mathsf{MDP}^{(\ell,n,m)}$ as follows
    \begin{equation}
        \begin{split}
            &\fC_{\ell,n,m,T}(\hat \mu^{1:K}) :=  \frac{1}{K}\sum_{k=1}^K \hat{W}_{2,\lambda}(\hat{\mu}^k,\hat{\nu}^k)^2 \\
            &\fC_{\ell,n,m,t}(\hat \mu^{1:K}) := \inf_{U\in\UU_m}\int_{\cP^{(\ell)}(\XX_n)^K}\eta^{(\ell,n,m)}(\de \hat{\mu}_{t+1}^{1:K}\mid\hat \mu^{1:K}; U)\,\fC_{\ell,n,m,t+1}(\mu_{t+1}^{1:K}),
        \end{split}
    \end{equation}
     for all $\hat \mu^{1:K}\in\cP^{(\ell)}(\XX_n)^K$ and all $t\in\cT$. Then there is a finite sequence of measurable functions $(\hat{\gamma}^*_{\ell,n,m,t})_{t\in\cT}$ from $\cP^{(\ell)}(\XX_n)^K$ to $\UU_m$ such that $\hat{\gamma}^*_{\ell,n,m,t}(\hat{\mu}^1,\ldots\hat{\mu}^K) = U_t^{(\ell,n,m)*}$ for all $t\in\cT$ and such actions are optimal in the sense that
    \begin{equation}
        \frac{1}{K}\sum_{k=1}^K W_{2,\lambda}(\hat{\mu}_T^{k,(U_t^{(\ell,n,m)*})_{t\in\cT}}, R_\ell(\nu^k))^2 = \inf_{(U_t^{(m)})_{t\in\cT}\in\UU_m^T}\frac{1}{K}\sum_{k=1}^K W_{2,\lambda}(\hat{\mu}_T^{k,(U_t^{(m)})_{t\in\cT}}, R_\ell(\nu^k))^2.
    \end{equation}
\end{theorem}

Theorem \ref{thm:triply-quantized-optimal-policies} yields us an optimal closed-loop policy from which we again obtain an optimal open-loop policy as we discuss in previous parts. So, the architecture of Transformer corresponding to $\mathsf{MDP}^{(\ell,n,m)}$ is now fixed. Now, we provide a learning algorithm to summarize the procedure.

\begin{algorithm}[H]
\caption{Triply Quantized Dynamic Programming}
\label{alg:triply}
\begin{algorithmic}[1]

\STATE \textbf{Input:} data set $\{(\mathbf{x}^k,\mathbf{y}^k)\}_{k=1}^K$, horizon $T$, quantization levels $(\ell,n,m)$

\STATE Construct state grid $\XX_n$ and define the quantizer $Q_n$

\STATE Construct finite action set $\UU_m = \{U^1,\dots,U^M\}$

\STATE Construct measure grid $\mathcal \cP^{(\ell)}(\XX_n)$ and quantizer $R_\ell$

\FOR{$t = T-1$ down to $0$}
    \FORALL{$(\hat{\mu}^1,\dots,\hat{\mu}^K)\in\cP^{\ell}(\XX_n)^K$}
        \STATE $\fC_{\ell,n,m,t}(\hat{\mu}^1,\dots,\hat{\mu}^K)
        =
        \min_{U \in \UU_m}
        \fC_{\ell,n,m,t+1}\big(
        \Phi^{(\ell,n,m)}(\hat{\mu}^1,U),\dots,\Phi^{(\ell,n,m)}(\hat{\mu}^K,U)
        \big)$

        \STATE $\gamma_t^*(\hat{\mu}^1,\dots,\hat{\mu}^K) = \argmin_{U \in \UU_m}\fC_{\ell,n,m,t+1}\big(
        \Phi^{(\ell,n,m)}(\hat{\mu}^1,U),\dots,\Phi^{(\ell,n,m)}(\hat{\mu}^K,U)
        \big)$
    \ENDFOR
\ENDFOR

\STATE \textbf{Forward pass:}

\FOR{$t=0,\dots,T-1$}
    \STATE $U_t^\star \leftarrow \gamma_t^*(\hat{\mu}_t^1,\dots,\hat{\mu}_t^K)$
    \FOR{$k=1,\dots,K$}
        \STATE $\hat{\mu}_{t+1}^k \leftarrow \Phi^{(\ell,n,m)}(\hat{\mu}_t^k, U_t^\star)$
    \ENDFOR
\ENDFOR

\STATE \textbf{Output:} $(U_t^\star)_{t=0}^{T-1}$

\end{algorithmic}
\end{algorithm}

We still need to be able to guarantee that any optimal policy for $\mathsf{MDP}^{(\ell,n,m)}$ performs good enough for the original problem. Until the end of this subsection, our aim is to show that this is indeed possible.

\begin{theorem}\label{thm:triply-quantized-efficiency}
    Let Assumptions \ref{assume:compactness}, \ref{assume:lambda} and $\ref{assume:growth-assumptions}$ hold. Then, there exists a sequence of nonnegative real numbers $(\varepsilon^{(\ell,n,m)})_{\ell,n,m\in\NN}$ with
    \begin{equation*}
     \varepsilon^{(\ell,n,m)}\xrightarrow{\ell,n,m\to\infty}0
    \end{equation*}
    such that for all $\ell,n,m\in\NN$ and all optimal controls $\underline{U}^{(\ell,n,m)*} = (U_{t,\ell,n}^{(\ell,n,m)*})_{t\in\cT} \in \UU_m^T$ for $\mathsf{MDP}^{(\ell,n,m)}$, we have
    \begin{equation*}
        V(\underline{U}^{(\ell,n,m)*}) \leq \inf_{\underline{\gamma}\in\Gamma}V(\underline{\gamma}) + \varepsilon_{\ell,n,m}.
    \end{equation*}
\end{theorem}
\begin{proof}
    See Appendix~\ref{proof:triply-quantized-efficiency}.
\end{proof}

We now know that an optimal policy for $\mathsf{MDP}^{(\ell,n,m)}$ is near-optimal for the original problem. Then, we can easily solve the dynamic programming equations provided in Theorem \ref{thm:triply-quantized-optimal-policies} and obtain optimal controls for $\mathsf{MDP}^{(\ell,n,m)}$. We know that they perform well enough for finer quantization, in other words, the model can be more accurate if the quantization levels are chosen higher.

\section{Robustness to Distributional Initialization Errors, Asymptotic
Consistency and $\Gamma$-convergence to Optimality}\label{sec:robustness}
In training of Transformers, the optimal control actions are chosen for a given data set. This begs the question "how would the trained Transformer perform for unseen data?". This questions is called the \emph{generalization problem}. In this section, we provide an answer to it by proving that the optimal controls are robust under distributional perturbations. In other words, as a sequence of empirical initial laws converges to the true distribution, we prove the continuity of value function under weak\textsuperscript{$*$}-topology.

For this part, we extend our formulation as follows: Given a sequence of data sets $\cD_r := \{(\boldsymbol{x}^{k,r}, \boldsymbol{y}^{k,r}) : k=1,2,\ldots,K_r\}$, where $r\in\NN$, we lift them to the empirical measures of the following form
\begin{equation*}
    \PP_r := \frac{1}{K_r}\sum_{k=1}^{K_r}\delta_{(\mu_0^{k,r}, \nu^{k,r})},
\end{equation*}
where $\mu_0^{k,r} := \frac{1}{N}\sum_{i=1}^N\delta_{x^{i,k,r}}$ and $\nu^{k,r} := \frac{1}{N}\sum_{i=1}^N\delta_{y^{i,k,r}}$ for all $r\in\NN$ and all $k\in\{1,\ldots,K_r\} =: \cK_r$. Hence, we see $\PP_r \in \cP(\cP(\XX)\times\cP(\XX))$. Now, we define the cost functional (with an abuse of notation), for all $\underline{U}\in\UU^T$ and $P\in \cP(\cP(\XX)\times\cP(\XX))$, as
\begin{equation*}
    V(P, \underline{U}) := \int_{\cP(\XX)\times\cP(\XX)}P(\de\mu,\de\nu)W_{2,\lambda}(\Phi_T(\mu, \underline{U}), \nu)^2,
\end{equation*}
where we end up with $\Phi_T:\cP(\XX)\times\UU^T\to\cP(\XX)$ in a recursive fashion, using
\begin{equation*}
    \Phi_{t}(\mu, \underline{U}|_{t}) = \Phi(\Phi_{t-1}(\mu, \underline{U}|_{t-1}), U_{t}), \quad \forall t\in\cT,
\end{equation*}
and $\Phi_{1}(\mu, \underline{U}|_1) = \Phi(\Phi(\mu, U_0), U_1)$. Here we use the notation $\underline{U} = (U_0,U_1,\ldots,U_{T-1})\in\UU^T$ and $\underline{U}|_t = (U_0,\ldots,U_t)$ for all $t\in\overline{\cT}$. In short, $V$ is the value function for the empirical measure of a given data set that passes the first marginal through the Transformer flow using a trajectory of control actions, and compares the resulting distribution with the second marginal. The cost functional here is defined by $\fc((\mu,\nu), \underline{U}^T) := W_{2,\lambda}(\Phi_T(\mu, \underline{U}), \nu)^2$ for all $(\mu,\nu) \in \cP(\XX)\times\cP(\XX)$ and all $\underline{U}\in\UU^T$. For an empirical measure, for example considering $\PP_r$, the value function collapses to
\begin{equation*}
    V(\PP_r, \underline{U}) := \frac{1}{K_r}\sum_{k=1}^{K_r} W_{2,\lambda}(\Phi_T(\mu_0^{k,r}, \underline{U}^T), \nu^{k,r})^2.
\end{equation*}
Therefore, this is indeed the right generalization since $\Phi_T(\mu_0^{k,r}, \underline{U}^T)$ corresponds to the terminal state under the Transformer dynamics as it is illustrated in Figure \ref{fig:phi-bold-map}. Moreover, this generalization is more powerful from the point of our analysis than the previous ensemble formulation we have discuss earlier. However, it is computationally less tractable because of the double lifting. Therefore, we prefer one lifting formulation when discussing the training problem. Because of the intrinsic properties of the problem, we have such nice equivalences that allows us to go back and forth between these formulations.

{\begin{remark}\label{rem:symmetry}
Let $\mathcal Z:=\mathcal P(\mathbb X)\times\mathcal P(\mathbb X)$ and let $\mathbf e[z]:=\frac1K\sum_k\delta_{(\mu^k,\nu^k)}$ for
$z=((\mu^k,\nu^k))_{k\in\mathcal K}\in\mathcal Z^K$. With
$\Psi(P,U):=(\Phi(\cdot,U)\times\mathrm{id})_{\#}P$ and
$J(P):=\int_{\mathcal Z}W_{2,\lambda}(\mu,\nu)^2\,dP$, we have

\[
    e[z](\Phi(\mu;U), \nu)=\Psi(e[z](\mu,\nu),U) \qquad \text{ and } \qquad C_T=J\circ e[z].
\]
Backward induction then
gives $C_t=\mathcal C_t\circ e[z]$, where
$\mathcal C_t(P):=\min_{U\in\mathbb U}\mathcal C_{t+1}(\Psi(P,U))$ and $\mathcal C_T:=J$.
Thus training is the special case $P=P_D$ of the double-lifted problem of
Section~\ref{sec:robustness}, whose state space does not depend on $K$, and
Theorem~\ref{thm:robustness} is the continuity of $\mathcal C_0=V^*$ in the initial state.
We keep the $K$-tuple only as the computational representation of a $K$-atom empirical
measure of input-output pairs. The reduction does not change the cost of the quantized scheme, since the dynamics
are deterministic and at most $|\mathbb U_m|^t$ states are reachable at step $t$ in either
representation.
\end{remark}}

We endow $\cP(\cP(\XX)\times\cP(\XX))$ with the induced weak\textsuperscript{$*$}-topology. That is, $P_j\to P$ in $\cP(\cP(\XX)\times\cP(\XX))$, if given a continuous and bounded map $\varphi:\cP(\XX)\times\cP(\XX)\to\RR$, we have
\begin{equation*}
    \int_{\cP(\XX)\times\cP(\XX)} P_j(\de\mu,\de\nu)g(\mu,\nu) \to \int_{\cP(\XX)\times\cP(\XX)} P(\de\mu,\de\nu)g(\mu,\nu).
\end{equation*}
Under Assumption \ref{assume:compactness}, the underlying space $\cP(\XX)\times\cP(\XX)$ is compact, thereby, $\cP(\cP(\XX)\times\cP(\XX))$ is also compact, and can be equipped with a metric generating the same topology. Note that we give $\cP(\cP(\XX)\times\cP(\XX))$ the respective Borel $\sigma$-field structure to make it a measurable space.

\begin{lemma}\label{lem:value-function-is-continuous}
    Under Assumption \ref{assume:compactness}, the value function $V$ is jointly continuous on $\cP(\cP(\XX)\times\cP(\XX))\times\UU^T$.
\end{lemma}
\begin{proof}
    See Appendix~\ref{proof:value-function-is-continuous}.
\end{proof}

This lemma is sufficient for us to prove the robustness of value function to varying distributions.

\begin{theorem}\label{thm:robustness}
    Let $\PP_r$ be the empirical distribution associated to the data set $\cD_r$ for all $r\in\NN$. Assume that $\PP_r \xrightarrow{\mathrm{w}^*} \PP$. Let $\underline{U}_r^*$ be a sequence of optimal control actions for $\PP_r$ in the sense that if we train the Transformer with the underlying data set, $\underline{U}_r^*$ is optimal. Also, let $\underline{U}^*$ be a sequence of optimal control actions for $\PP$ in the sense that $V^*(\PP) := \inf_{\underline{A}\in\UU^T}V(\PP, \underline{A}) = V(\PP, \underline{U})$. Then $V(\PP, \underline{U}_r^*) \to V^*(\PP)$.
\end{theorem}
\begin{proof}
    See Appendix~\ref{proof:robustness}.
\end{proof}

In principle, given a sequence of data sets $\cD_r$ that are qualitatively getting better, in the sense that they get close to the true distribution, then the resulting control actions are near optimal for the true distribution. If the data used for training comes from an i.i.d. sequence under a fixed probability measure; then the policies obtained for finitely many samples are asymptotically near optimal for the limit problem (with infinite training data). Furthermore, the sequence of generated actions converges, possibly along a subsequence, to an optimal policy for the limit problem, which is often termed as $\Gamma$-convergence.

\begin{corollary}\label{cor:LimitPoint}
    Let $\PP_r\xrightarrow{\mathrm{w}^*} \PP$ in $\cP(\cP(\XX))$. With the notation of Theorem \ref{thm:robustness} and under Assumption \ref{assume:compactness}, if $\underline{U}_r^* = (U_{t,r}^*)_{t\in\cT} \in \argmin V(\PP_r, \cdot)$ for all $r\in\NN$, then $\{\underline{U}_r^* : r\in\NN\}$ has a limit point $\underline{U}^* = (U_t^*)_{t\in\cT}$ such that $\underline{U}^* \in \argmin V(\PP, \cdot)$.
\end{corollary}
\begin{proof}
    See Appendix~\ref{proof:LimitPoint}.
\end{proof}

\section{Numerical Experiments}
{ To illustrate the convergence claim of Theorem~\ref{thm:triply-quantized-efficiency}, we set up a relatively simple experiment in which we aim to approximate a function $T(x_1,\ldots,x_N) = (y_1,\ldots,y_N)$ where $y_i = \sum_{j=1}^N\mathrm{softmax}_j(0.3\langle x_i, x_j\rangle)x_j$ for each $i\in\cN$. The programs are implemented in {\tt Python} using the {\tt JAX} library \citep{jax2018github} according to Algorithm~\ref{alg:triply}. We have utilized automatic vectorization and just-in-time compilation features to reduce the computational complexity caused by the programming language itself. The results are given and discussed in the following figures.

In the experiment below, we use the following configuration: $K_{\rm train}=35$, $K_{\rm test} = 15$, $N=4$, $n=10$, $\ell=20$, $T=2$, $\beta=0.5$, $\lambda=500$, where $K_{\rm \bullet}$ denotes the data count used in $\bullet$ing with $\bullet\in\{\text{train}, \text{test}\}$. Light colored regions illustrate the interquartile range, while the dark colored lines are the averages.

\begin{figure}[h!]
    \centering
    \includegraphics[width=0.55\linewidth]{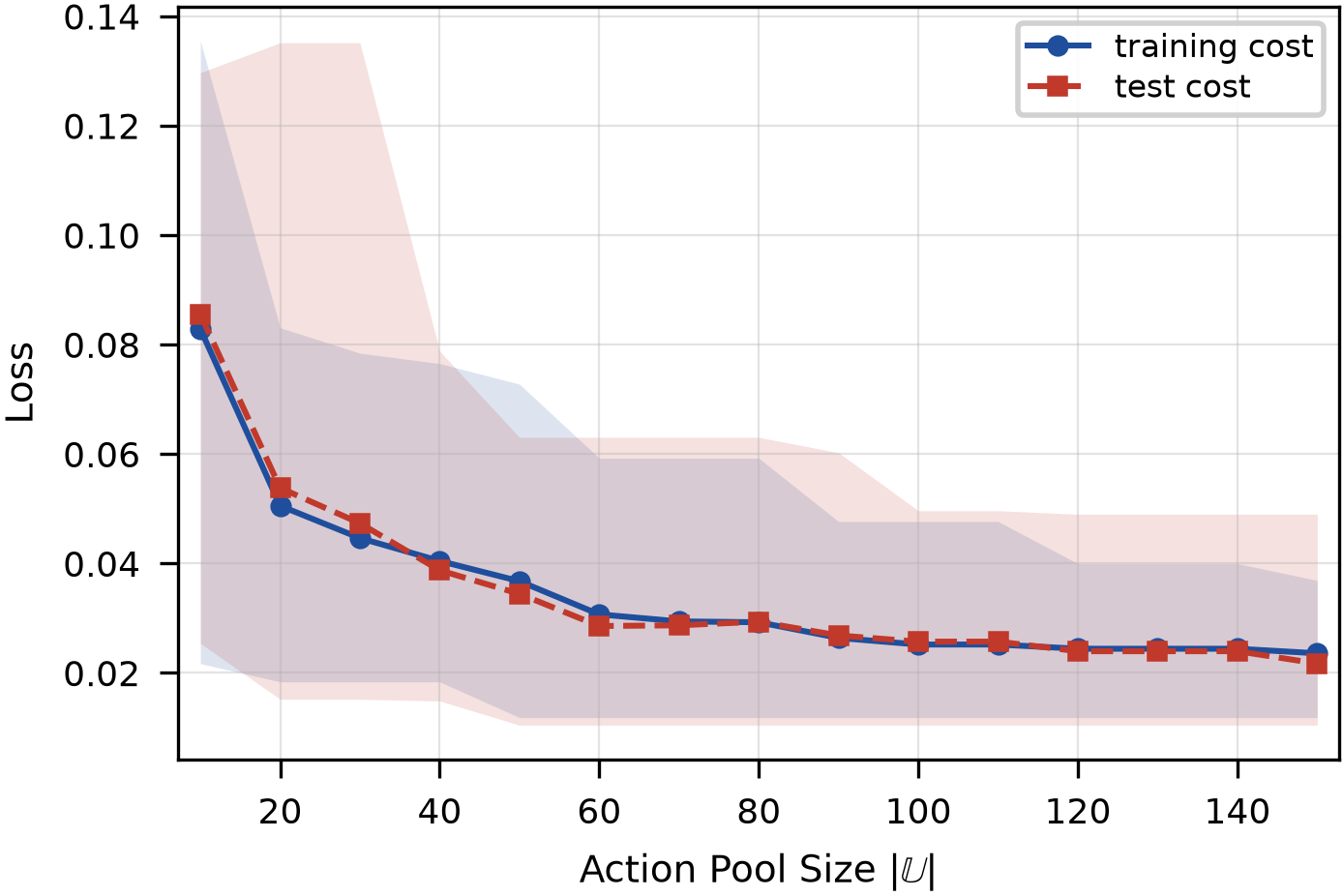}
    \caption{Synthetic toy experiment showing the convergence of the dynamic programming based training on the triply quantized method. }
    \label{fig:synthetic}
\end{figure}

For fixed state and measure-state quantization levels, the decrease in training and test losses are evident. However, one limitation is worth mentioning. Due to the curse of dimensionality, we work with limited data, low parameter counts and with non-full-rank controls.

\begin{figure}[H]
    \centering
    \includegraphics[width=0.9\linewidth]{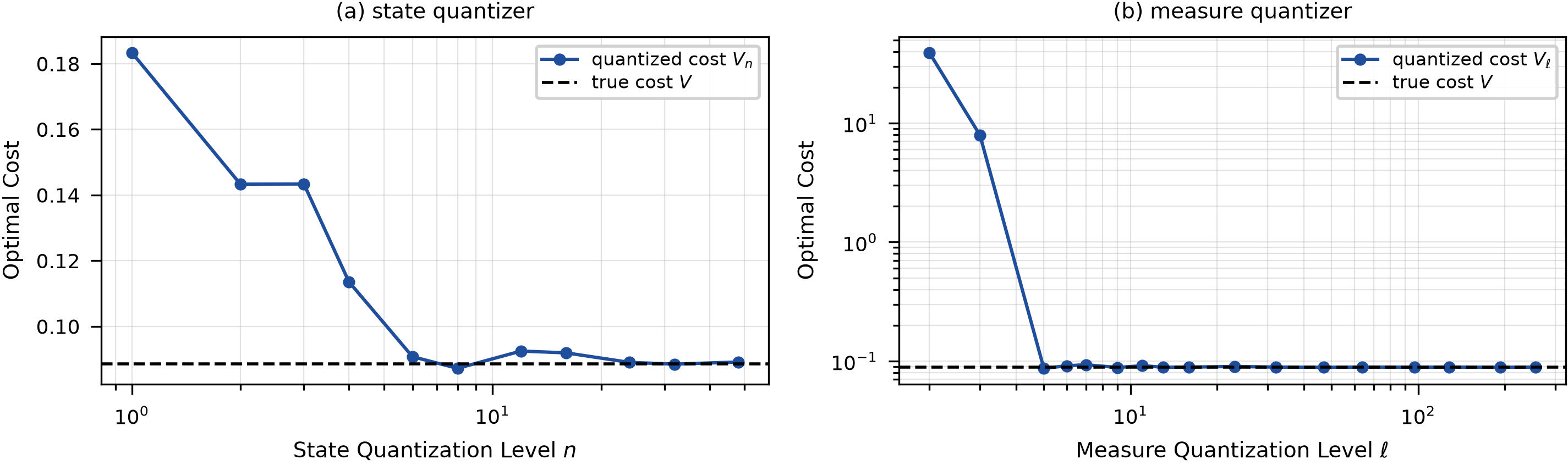}
    \caption{The figure on the left is obtained by refining the state quantization level $n$, while the figure on the right is an outcome of refining the measure quantization level $\ell$. }
    \label{fig:ellandn}
\end{figure}}

{ 
In the above experiment, we use the following  configuration: ($m=10$, $K=25$, $N=4$, $T=2$), the cost functional of quantized models converge to the cost functional of the base model.

The results show that the quantized framework is able to decrease the error as the quantization rates get finer, as Theorem~\ref{thm:triply-quantized-efficiency} asserts. One limitation here is worth mentioning: As an optimization step, before dynamic programming, we construct all possible states given a set of actions using the measure-valued Transformer dynamics, given by $\Phi$. This allows us to discard all the irrelevant quantized measures in $\cP_\ell(\XX_n)$, therefore, limiting the effects of the state and measure-state quantization levels. However, when the action quantization level is fixed, Figure~\ref{fig:ellandn} shows that as the state-related quantization levels grow, the difference between the quantized expected cost functional and the true cost shrinks.}

\section{Conclusion}
In this paper, we formulated Transformers as a discrete-time McKean---Vlasov dynamics system and lifted this dynamics to a measure-valued discrete-time Markov decision process. This lifting yielded the Markov property for the ensemble of probability measures that are obtained from initial data, which allowed us to use the dynamic programming principle to establish near-optimality. Although the optimal policies obtained from dynamic programming principle are of closed-loop nature, we argued that they should be seem as a theoretical tool, but not a practical one. Then we showed that from the deterministic flow of measure-valued MDP and the deterministic nature of closed-loop policies, we can get initial-data-dependent open-loop policies, which are compatible with the orthodox training methods of Transformers by fixing neural network weights after training. Therefore, such policies are the correct abstraction of neural network weights from the perspective of optimal control theory. After that, we proposed a quantization-based learning algorithm, which is compatible with the theory presented in this paper, and we showed that it is possible to get near-optimal policies that are computationally more feasible than running dynamic programming on the infinite-dimensional measure-valued space.

Although, by quantizing, we reduced the infinite-dimensional measure-valued space to a finite-dimensional probability simplex, then to a finite set, the training-by-quantization is not designed as a scalable solver of this problem, but as a helper to conduct numerical experiment that is compatible with the theory. The main contribution of the framework presented is providing another perspective on Transformers to understand its structure better and guarantee the existence of optimal weights for it, rather than providing an efficient and scalable algorithm. Also, the subjects of this paper are not a competitor to gradient-descent-based techniques, however; since many applications of Transformers use gradient-descent methods for, in general, a non-convex and non-differentiable system, which indicates that the underlying structure of Transformers is not really understood. 
\subsection{Limitations}
{ We emphasize two limitations of our paper.
\smallskip

\noindent{\it Scalability.} The main limitation of our approach is lack of scalability due to the curse of dimensionality. It is infeasible to search of a grid of network weights, even for low precision levels and layer counts. Therefore, we consider enhancing our formulation with scalable training algorithms.
\smallskip

\noindent{\it Architecture.} Our formulation lacks some of the components that a practical Transformer model has. For example, causal masks are widely used in next-token prediction problems. However, causal masking of the form $\mathbf{1}\{p_x \leq p_z\}$ for some $x,z\in\XX$ is a discontinuous map. So, the weak Feller property of the transition kernel $\eta$ fails. It is possible to choose a smooth causal mask to overcome this issue, but that changes the model. Moreover, our Transformer model uses fixed input sequence length $N$ and fixed positional encodings, instead of learnable ones.}

\appendix

\section{Auxiliary Results}
In this part, we provide all of the necessary lemmata and {related} results that are used to prove the main results of this section.

\begin{lemma}\label{lem:IntegralContinuity}
    Let $\sY$ and $\sZ$ be compact metric spaces. For some $m\in\NN$, if $h:\sY\times\sZ\to\RR^m$ is jointly continuous, then the function defined by
    \[
        \sY\times\cP(\sZ)\ni(y, \mu)\mapsto \int_\sZ \mu(\de z)\, h(y, z)
    \]
    is jointly continuous.
\end{lemma}
\begin{proof}
    Let $\mu_n\to\mu$ in $\cP(\sZ)$ and $y_n\to y$ in $\sY$. Then
    \begin{align*}
        \left\|\int_\sZ \de\mu_n\, h(y_n, \cdot) - \int_\sZ \de\mu\, h(y, \cdot)\right\| &= \bigg\|\int_\sZ \de\mu_n\, h(y_n, \cdot) - \int_\sZ \de\mu_n\, h(y, \cdot)\\
        &\qquad\qquad+ \int_\sZ \de\mu_n\, h(y, \cdot) - \int_\sZ \de\mu\, h(y, \cdot)\bigg\|\\
        &\leq\bigg\|\int_\sZ \de\mu_n\, h(y_n, \cdot) - \int_\sZ \de\mu_n\, h(y, \cdot)\bigg\|\\
        &\qquad\qquad+\bigg\|\int_\sZ \de\mu_n\, h(y, \cdot) - \int_\sZ \de\mu\, h(y, \cdot)\bigg\|\\
        &\leq \sup_{z\in\sZ}\|h(y_n, z) - h(y, z)\| + \int_\sZ \de(\mu_n - \mu)\,\|h(y,\cdot)\|.
    \end{align*}
    As $n\to+\infty$, in the last inequality, the first term vanishes due to uniform continuity of $h$. For the second term, we know that $\mu_n \to \mu$ in weak$^*$-topology, therefore, for any $g\in\cC_b(\cZ)$, we have
    \[
        \int_\cZ \mu_n(\de z)\, g(z) \rightarrow \int_\cZ \mu(\de z)\, g(z).
    \]
    Since $\|h(y, \cdot)\|$ is a continuous function on a compact space $\sZ$, then $z\mapsto\|h(y,\cdot)\|\in\cC_b(\cZ)$, which proves the vanishing of second term as $n\to+\infty$. 
\end{proof}
\def\diam{\mathrm{diam}}

\begin{lemma}\label{lemma:f-lipschitz}
    Under Assumption \ref{assume:compactness}, $f$ satisfies the following Lipschitz condition for all $x,y\in\XX$, $U\in\UU$, and all $\mu,\nu\in\cP(\XX)$
    \begin{equation*}
        \|f(x,U,\mu)-f(y,U,\nu)\|_{2,\lambda} \leq L_f(\|x-y\|_{2,\lambda} + W_{2,\lambda}(\mu,\nu))
    \end{equation*}
    for some $L_f \geq 0$.
\end{lemma}
{
    \begin{proof}
    Let $h(x,z) := \e^{\beta\langle Qx, Kz\rangle}$, then writing $H := \e^{\beta(\diam(\XX)^2\diam(\UU)^2)}$, we have $H^{-1} \leq h \leq H$. Moreover, $\nabla_x h = \beta h Q^\top Kz$, and therefore, $\|\nabla_x h\| \leq \beta H\diam(\UU)^2\diam(\XX) =: \ell_0$. Hence, $h$ is $\ell_0$-Lipschitz.

    We also define $k(x,z) := h(x,z)Vz$. Then,
    \begin{itemize}
        \item $k$ is $\ell_1$-Lipschitz in $x$ with $\ell_1 = \ell_0\diam(\XX)\diam(\UU)$,
        \item $k$ is $\ell_2$-Lipschitz in $z$ with $\ell_2 = \ell_0\diam(\UU)\diam(\XX) + H\diam(\UU)$.
    \end{itemize}

    Let $\pi$ be an optimal $W_{2,\lambda}$-coupling of $\mu$ and $\nu$. Then, for the numerator term $N$ in the self-attention block, we have
    \begin{align*}
        \|N(x,\mu) - N(y,\nu)\| &\leq \int_{\XX\times\XX}\pi(\de z, \de z')\,\|k(x,z) - k(y,z')\|\\
        &\leq \ell_1\|x-y\| + \ell_2\max\{1, 1/\sqrt{\lambda}\}W_{2,\lambda}(\mu,\nu).
    \end{align*}
    Since $z \mapsto h(y,z)$ is $\ell_0$-Lipschitz, same argument applied to $D$ yields
    \[
        |D(x,\mu) - D(y, \nu)| \leq \ell_0(\|x-y\| + \max\{1, 1/\sqrt{\lambda}\}W_{2,\lambda}(\mu,\nu)).
    \]
    Then, for the quotient, we have
    \begin{align*}
        \bigg\|\frac{N(x,\mu)}{D(x,\mu)} - \frac{N(y,\nu)}{D(y,\nu)}\bigg\| &= \bigg\|\frac{N(x,\mu)D(y,\nu) - D(x,\mu)N(y,\nu)}{D(x,\mu)D(y,\nu)}\bigg\|\\
        &\leq\bigg\|\frac{N(x,\mu)D(y,\nu) - N(x,\mu)D(x,\mu)}{D(x,\mu)D(y,\nu)}\bigg\|\\
        &\qquad\qquad+\bigg\|\frac{N(x,\mu)D(x,\mu) - D(x,\mu)N(y,\nu)}{D(x,\mu)D(y,\nu)}\bigg\|\\
        &\leq \|N(x,\mu)\|\bigg|\frac{D(y,\nu) - D(x,\mu)}{D(x,\mu)D(y,\nu)}\bigg|\,+\,\frac{\|N(x,\mu) - N(y,\nu)\|}{|D(y,\nu)|}\\
        &\leq B_NH^2|D(y,\nu) - D(x,\mu)| + H\|N(x,\mu) - N(y,\nu)\|
    \end{align*}
    where $B_N$ is the sup-norm of $N$. Together with
    \[
        \|W\sigma(Ax + b) - W\sigma(Ay + b)\| \leq \diam(\UU)^2\mathrm{Lip}(\sigma)\|x-y\|,
    \]
    and by defining $\tilde{f}(x, U, \mu) := x + W\sigma(Ax + b) + N(x,\mu)/D(x,\mu)$, we get
    \[
        \|\tilde{f}(x, U, \mu) - \tilde{f}(y, U, \nu)\| \leq L_{\tilde{f}}(\|x-y\| + W_{2,\lambda}(\mu,\nu))
    \]
    where 
    \[
        L_g = \max\{1, 1/\sqrt{\lambda}\}(\diam(\UU)^2\mathrm{Lip}(\sigma) + H(\ell_0 + \ell_1) + H^2 B_N\ell_0)
    \]
    and
    \[
        g(x, U, \mu) := W\sigma(Ax+b) + \frac{N(x, U, \mu)}{D(x, U, \mu)}.
    \]
    Finally, with the identity $\sqrt{a^2+b^2} \leq a + b$ (for $a,b>0$) and by Cauchy-Bunyakovski-Schwarz inequality,
    \[
        \|\tilde{f}(x, U, \mu) - \tilde{f}(y, U, \nu)\| \leq L_f(\|x-y\| + W_{2,\lambda}(\mu,\nu))
    \]
    where $L_f = \sqrt{2}\max\{1, L_{\mathbf{P}}\}+L_{\mathbf{P}}L_g$.
\end{proof}
}

\begin{corollary}\label{cor:phi-lipschitz}
    For any $\mu,\nu\in\cP(\XX)$ and $U\in\UU$, we have
    \begin{equation*}
        W_{2,\lambda}(\Phi(\mu,U), \Phi(\nu, U)) \leq L_\Phi W_{2,\lambda}(\mu,\nu)
    \end{equation*}
    for some $L_\Phi \geq 0$. We here write $\boldsymbol{\mu} = (\mu^1,\ldots,\mu^K)$ and $\bm{\nu} = (\nu^1,\ldots,\nu^K)$.
\end{corollary}
\begin{proof}
    Let $\pi^*$ be an $W_{2,\lambda}$-optimal coupling of $\mu$ and $\nu$. Define
    \begin{equation}\label{eq:anotherCoupling}
        \tilde{\pi}(\de x, \de y) := \pi^*(f^{-1}(\de x, U, \mu), f^{-1}(\de y, U, \nu)),
    \end{equation}
    which is a coupling for $\Phi(\mu, U)$ and $\Phi(\nu, U)$. { Note that $f^{-1}$ in~\eqref{eq:anotherCoupling} is the pullback map.}

    By Lemma \ref{lemma:f-lipschitz} and by the elementary inequality $(a+b)^2 \leq 2a^2 + 2b^2$, we get the following sequence of inequalities:
    \begin{align*}
        W_{2,\lambda}(\Phi(\mu,U), \Phi(\nu, U))^2 &\leq \int_{\XX\times\XX} \tilde{\pi}(\de x,\de y) \| x- y\|_{2,\lambda}^2\\
        &\leq\int_{\XX\times\XX}\pi^*(\de x, \de y)\|f( x,U,\mu)-f( y;U;\nu)\|_{2,\lambda}^2\\
        &\leq 2L_f^2\int_{\XX\times\XX}\pi^*(\de x, \de y)(\| x- y\|_{2,\lambda}^2 + W_{2,\lambda}(\mu,\nu)^2)\\
        &=4L_f^2 W_{2,\lambda}(\mu,\nu)^2.
    \end{align*}
    Thus, writing $L_\Phi := 2L_f$, we get the claim.
\end{proof}

\begin{lemma}\label{lemma:approximate-measure}
    Let $\mu\in\cP(\XX)$. { Whenever $1/n < \delta$}, it follows that $W_{2,\lambda}(\mu,\mu\circ (Q_n)^{-1}) < \frac{1}{n}$.
\end{lemma}
\begin{proof}
    Let $\pi$ be a coupling for the law of $(x, Q_n(x))$ with $x\sim\mu\in\cP(\XX)$. Then,
    \begin{align*}
        W_{2,\lambda}(\mu,\mu\circ (Q_n)^{-1})^2 &\leq \int_{\XX\times\XX}\pi(\de x, \de  y)\|x- y\|_{2,\lambda}^2\\
        &\leq \EE_{x\sim\mu}\|x-Q_n(x)\|_{2,\lambda}^2\\
        &\leq \int_{\XX\times\XX}\pi(\de x, \de  y)\frac{1}{n^2}.
    \end{align*}
    In the last inequality, we use $\|x-Q_n(x)\|_{2,\lambda} \leq \frac{1}{n}$, given by the definition of $(Q_n)^{-1}$. Thus, $W_{2,\lambda}(\mu,\mu\circ (Q_n)^{-1}) \leq \frac{1}{n}$. { As a note, the assumption that $1/n < \delta$ ensures that the positions are preserved: any state with distance at most $1/n$ must have the same position.}
\end{proof}

\section{Proofs of the Results in Section~\ref{sec:lifting-section}}

\subsection{Proof of Proposition~\ref{prop:Continuity}}\label{proof:Continuity}
    \noindent{\it Proof of (i).} It is straightforward to see that $\XX\ni x \mapsto W\sigma(Ax+b)$ is continuous. Since continuous functions form a vector space, it suffices to analyze the remaining term. Since $\XX$ and $\UU$ are compact, then for any $x\in\XX$ and $U = (W, A, b, Q, K, V)\in\UU$, the function $\XX\ni z\mapsto\e^{\beta\inner{Qx}{K\tilde{z}}}$ is bounded. Also, it is continuous. Therefore, for any sequence $(P_n)_{n\in\NN}$ in $\cP(\XX)$ $\mathrm{w}^*$-converging to $P^*\in\cP(\XX)$ we have
    \begin{equation*}
        \int_\XX P_n(\de\tilde{z})\,\e^{\beta\inner{Qx}{K\tilde{z}}} \xrightarrow{n\to\infty}\int_\XX P^*(\de\tilde{z})\,\e^{\beta\inner{Qx}{K\tilde{z}}}.
    \end{equation*}
    This convergence is uniform on compact subsets of the action space, which suffices for the entire action space under Assumption \ref{assume:compactness}.  On the other hand, since $\XX\ni z\mapsto\e^{\beta\inner{Qx}{K\tilde{z}}}$ is bounded and we are working on finite measures, dominated convergence theorem \cite[Theorem 2.4.5]{cohn2013measure} is applicable. Therefore, for any $Q_n\to Q$ in $\RR^{d\times k}$, $K_n\to K$ in $\RR^{d\times k}$, $x_n\to x$ in $\XX$, we have
    \begin{equation*}
        \int_\XX P(\de\tilde{z})\,\e^{\beta\inner{Q_nx_n}{K_n\tilde{z}}} \xrightarrow{n\to\infty}\int_\XX P(\de\tilde{z})\,\e^{\beta\inner{Qx}{K\tilde{z}}}.
    \end{equation*}
    Note that this convergence is not dependent on the choice of measure $P$. Since we have uniform convergence in each separate argument, we have joint continuity.

    Next, we examine the continuity of $(\mu, Q, K, V, x) \mapsto \int_\XX \mu(\de \tilde{z}) \e^{\beta\inner{Qx}{K\tilde{z}}}V\tilde{z}$. It is important to see that the integration here is in the sense of Bochner integration in finite dimensions. The strong measurability of $(Q, K, V, x)\mapsto\e^{\beta\inner{Qx}{K\tilde{z}}}V\tilde{z}$ is routine and the continuity of the same function is an application of dominated convergence theorem for Bochner integration \cite[Theorem E.6]{cohn2013measure}. Since the continuous functions form an algebra, $f$ is (uniformly) continuous in the control argument $U$. Since $f$ is uniformly continuous in all its arguments, $f$ is jointly continuous on its domain.
    \smallskip

    \noindent{\it Proof of (ii).} { Let $(\mu_n, U_n) \to (\mu,\nu)$ in $\cP(\XX)\times\UU$ and $g:\XX\to\RR$ be $1$-Lipschitz. Then,
    \begin{align*}
        \bigg|\int_\XX g\,\de\Phi(\mu_n, U_n) - \int_\XX g\,\de\Phi(\mu, U)\bigg| &= \bigg|\int_\XX \mu_n(\de x)\,g(f(x, U_n, \mu_n)) - \int_\XX \mu(\de x)\, g(f(x, U, \mu))\bigg|\\
        &\leq \bigg|\int_\XX \mu_n(\de x)\,g(f(x, U_n, \mu_n)) - \int_\XX \mu(\de x)\,g(f(x, U_n, \mu_n))\bigg|\\
        &\quad+\bigg|\int_\XX \mu(\de x)\,g(f(x, U_n, \mu_n)) - \int_\XX \mu(\de x)\,g(f(x, U, \mu))\bigg|\\
        &\leq \bigg|\int_\XX (\mu_n-\mu)(\de x)\, g(f(x, U_n, \mu_n))\bigg|\\
        &\qquad+\sup_{x\in\XX}|g(f(x, U_n, \mu_n)) - g(f(x, U, \mu))|.
    \end{align*}
    The first term in the last bound goes to zero since (a) $\mu_n\to\mu$ in $\mathrm{w}^*$-topology and the integrand continuously converges to $g(f(x, U, \mu)$ and thus the result holds by the generalized dominated convergence theorem \cite[Theorem 3.5]{serfozo1982convergence} (see also \cite{Lan81}) and (b) bounded and Lipschitz functions are convergence determining for the weak$^*$ topology, and $g$ is $1$-Lipschitz by assumption while $f$ is $L_f$-Lipschitz by Lemma~\ref{lemma:f-lipschitz}. The second term converges to zero since $g$ is $1$-Lipschitz and $f$ is uniformly continuous in $(U,\mu)$ by the Part (i).}\hfill$\blacksquare$

\subsection{Proof of Corollary~\ref{cor:transition-kernel}}\label{proof:transition-kernel}
    Due to Proposition \ref{prop:Continuity}-(ii), $\Phi$ is jointly continuous. Therefore, so is $\boldsymbol{\Phi}$. Then, let $(\boldsymbol{\mu}_n)_{n\in\NN}$ be a sequence in $\cP(\XX)^K$ converging to $\boldsymbol{\mu}\in\cP(\XX)^K$ in weak$^*$-topology and $(U_n)_{n\in\NN}$ be another sequence in $\UU$ converging to $U\in\UU$. Then for any continuous and bounded $g:\cP(\XX)^K\to\RR$, we have
    \begin{align*}
        \int_{\cP(\XX)^K} \eta(\de\boldsymbol{\mu}_{t+1}\mid \boldsymbol{\mu}_n; U_n)\,g(\boldsymbol{\mu}_{t+1}) &= \int_{\cP(\XX)^K}\delta_{\boldsymbol{\Phi}(\boldsymbol{\mu}_n, U_n)}(\de\boldsymbol{\mu}_{t+1})\,g(\boldsymbol{\mu}_{t+1}) \\
        &= g(\boldsymbol{\Phi}(\boldsymbol{\mu}_n, U_n))
    \end{align*}
    for all $n\in\NN$. Since both $g$ and $\Phi$ are continuous, we have
    \begin{equation*}
        \lim_{n\to\infty}g(\boldsymbol{\Phi}(\boldsymbol{\mu}_n, U_n)) = g(\boldsymbol{\Phi}(\boldsymbol{\mu}, U)) = \int_{\cP(\XX)^K} \eta(\de\boldsymbol{\mu}_{t+1}\mid \boldsymbol{\mu}; U)\,g(\boldsymbol{\mu}_{t+1}).
    \end{equation*}
    Thus, $\eta$ is weak Feller.\hfill$\blacksquare$

\subsection{Proof of Theorem~\ref{thm:equivalence}}\label{proof:equivalence}
Let $\boldsymbol{\mu}_{T,\underline{\gamma}} = (\mu_{T,\underline{\gamma}}^1,\ldots,\mu_{T,\underline{\gamma}}^K)$ be a generic terminal state evolved with $\boldsymbol{\Phi}$ under strategy $\underline{\gamma}\in\Gamma$ and initialized from the data set $\cD$. Respectively, let $\bm{\nu} := (\nu^1,\ldots,\nu^K)$ be the lifted labels, that is $\nu^k := \frac{1}{N}\sum_{i=1}^N\delta_{(y^{i,k})}$ for all $k\in\cK$, where $y^{i,k} = (p_i, s_y^{i,k})$ for some $s_y^{i,k}\in S$.

In our case, where we compare empirical measures with the same number of atoms, the Wasserstein distance reduces to finding an optimal permutation. In particular, the optimal cost functional takes the form
\begin{align*}
    V^*(\cD) &:= \inf_{\underline{\gamma}\in\Gamma}\frac{1}{K}\sum_{k=1}^K W_{2,\lambda}(\mu_{T,\underline{\gamma}}^k,\nu^k)\\
    &=\inf_{\underline{\gamma}\in\Gamma}\frac{1}{K}\sum_{k=1}^K \inf_{\sigma\in S_N}\frac{1}{N}\sum_{i=1}^N(\|s_{T,\underline{\gamma}}^{i,k} - s_y^{\sigma(i), k}\|_2^2 + \lambda|p_i - p_{\sigma(i)}|^2).
\end{align*}
{ If $\sigma(i) \neq i$, we want to ensure that
\[
    \lambda|p_i - p_{\sigma(i)}|^2 \geq \|s_{T,\underline{\gamma}}^{i,k} - s_y^{\sigma(i), k}\|_2^2.
\]
Since $|p_i - p_{\sigma(i)}|^2 > \delta^2$, to cover all states, it suffices to take $\lambda > \operatorname{diam}(S)^2/\delta^2$.\hfill$\blacksquare$}

\subsection{Proof of Theorem~\ref{thm:Existence}}\label{proof:Existence}

By Corollary~\ref{cor:transition-kernel}, $\eta$ is weak Feller. In addition, since the $W_{2,\lambda}$ metric is continuous, the conditions of Theorem~3.2.1 in \citep{hernandez2012discrete} are satisfied. Therefore, the minizing measurable selectors exist. \hfill $\blacksquare$.

\section{Proofs of the Results in Section~\ref{sec:quantization}}

\subsection{Proof of Proposition~\ref{prop:resnik-quantizer-error}}\label{proof:resnik-quantizer-error}
This is a simple corollary of \cite[Proposition 2]{reznik2011algorithm} in which we replace $d_2$ with the cost structure of $W_{2,\lambda}$ metric and this is equivalent on $\XX_n$. Then we have a bound on the difference
\begin{equation*}
    W_{2,\lambda}(\hat{\mu}, R_\ell(\hat{\mu})) \leq \frac{1}{\ell} \sqrt{\frac{\lfloor|\XX_n|/2\rfloor(|\XX_n| - \lfloor|\XX_n|/2\rfloor)}{|\XX_n|}} \leq \frac{\sqrt{|\XX_n|}}{\ell},
\end{equation*}
from which the desired result follows.\hfill$\blacksquare$

\subsection{Proof of Theorem~\ref{thm:triply-quantized-efficiency}}\label{proof:triply-quantized-efficiency}

We utilize the following processes that helps in intermediate steps of this proof
    \begin{itemize}
        \item[(N1)] $\mu_{t+1}^k = \Phi(\mu_t^k, U_t)$ with cost functional $V(\underline{U}) := \frac{1}{K}\sum_{k=1}^K W_{2,\lambda}(\mu_T^{k,\underline{U}}, \nu^k)^2$ where $\underline{U}\in\UU^T$ and $V^* = \inf_{\UU^T}V$.
        \item[(N2)] $\bar{\mu}_{t+1}^k = \Phi^{(n)}(\bar{\mu}_{t}^k, U_t)$ with $V_n(\underline{U}) := \frac{1}{K}\sum_{k=1}^K W_{2,\lambda}(\bar{\mu}_{T}^{k,\underline{U}}, \nu^k)^2$ where $\underline{U}\in\UU^T$ and $V_n^* = \inf_{\UU^T}V_n$.
        \item[(N3)] $\hat{\mu}_{t+1}^k = \Phi^{(\ell,n,m)}(\hat{\mu}_{t}^k, U_t^{(m)})$ with $V^{(\ell,n,m)}(\underline{U}^{(m)}) := \frac{1}{K}\sum_{k=1}^K W_{2,\lambda}(\hat{\mu}_{T}^{k,\underline{U}^{(m)}}, R_\ell(\nu^k))^2$ where $\underline{U}^{(m)}\in\UU_m^T$ and $V^{(\ell,n,m)*} = \inf_{\UU_m^T}V^{(\ell,n,m)}$.
    \end{itemize}
    for all $t\in\overline{\cT}$, $k\in\cK$, $(U_t)_{t\in\cT}\in\UU^T$, $(U_t^{(m)})_{t\in\cT}\in\UU_m^T$. Let $\underline{U}^{(\ell,n,m)*} = (U_{t,\ell,n}^{(\ell,n,m)*})_{t\in\cT} \in \UU_m^T$ be a sequence of optimal control actions for (N3). Then, we compose the objective in the following way
    \begin{align*}
        |V(\underline{U}^{(m),*}) - V^*| &\leq |V(\underline{U}^{(m),*}) - V_{n}(\underline{U}^{(m),*})| \\
        &+|V_n(\underline{U}^{(m),*}) - V^{(\ell,n,m)}(\underline{U}^{(m),*})| \\
        &+|V^{(\ell,n,m)*} - V_n^*|\\
        &+|V_n^* - V^*|
    \end{align*}
    \noindent\textbf{Step-1:} In $|V(\underline{U}^{(m),*}) - V_{n}(\underline{U}^{(m),*})|$, we only have error contribution from the state quantization. We prove a more general claim
    \begin{equation*}
        \sup_{\underline{U}\subset\UU^T}|V(\underline{U}) - V_n(\underline{U})| \xrightarrow{n\to\infty} 0.
    \end{equation*}
    Fix any $\underline{U}\subset\UU^T$, then
    \begin{align*}
        |V(\underline{U}) - V_n(\underline{U})| &= \bigg|\frac{1}{K}\sum_{k=1}^K W_{2,\lambda}(\mu_T^{k,\underline{U}}, \nu^k)^2 - W_{2,\lambda}(\hat{\mu}_T^{k,\underline{U}}, \nu^k)^2\bigg|\\
        &\leq \frac{1}{K}\sum_{k=1}^K |W_{2,\lambda}(\mu_T^{k,\underline{U}}, \nu^k) - W_{2,\lambda}(\hat{\mu}_T^{k,\underline{U}}, \nu^k)|\\
        &\hspace{12em}\cdot|W_{2,\lambda}(\mu_T^{k,\underline{U}}, \nu^k) + W_{2,\lambda}(\hat{\mu}_T^{k,\underline{U}}, \nu^k)|
    \end{align*}
    For each $k\in\cK$, using reverse triangle inequality we have
    \begin{equation*}
        |W_{2,\lambda}(\mu_T^{k,\underline{U}}, \nu^k) - W_{2,\lambda}(\hat{\mu}_T^{k,\underline{U}}, \nu^k)| \leq W_{2,\lambda}(\mu_T^{k,\underline{U}}, \hat{\mu}_T^{k,\underline{U}}).
    \end{equation*}
    For the second term
    \begin{align*}
        W_{2,\lambda}(\mu_T^{k,\underline{U}}, \nu^k)^2 &= \inf_{\pi\in\Pi(\mu_T^{k,\underline{U}}, \nu^k)}\int_{\XX\times\XX}\pi(\de (p,x), \de  y) (\|x-y\|_2^2 + \lambda |p-q|^2)\\
        &\leq \inf_{\pi\in\Pi(\mu_T^{k,\underline{U}}, \nu^k)}\int_{\XX\times\XX}\pi(\de (p,x), \de  y)(\operatorname{diam}(S)^2 + \lambda)\\
        &= \operatorname{diam}(S)^2 + \lambda < +\infty.
    \end{align*}
    Therefore, we obtain
    \begin{equation*}
        |V(\underline{U}) - V_n(\underline{U})| \leq \frac{2}{K}\sqrt{\operatorname{diam}(S)^2 + \lambda} \sum_{k=1}^K W_{2,\lambda}(\mu_T^{k,\underline{U}}, \hat{\mu}_T^{k,\underline{U}}).
    \end{equation*}
    Now the focus is on $W_{2,\lambda}(\mu_t^{k,\underline{U}}, \hat{\mu}_t^{k,\underline{U}})$ for which the relation
    \begin{align*}
        W_{2,\lambda}(\mu_{t+1}^{k,\underline{U}}, \hat{\mu}_{t+1}^{k,\underline{U}}) &= W_{2,\lambda}(\Phi(\mu_{t}^{k,\underline{U}}, U_t), \Phi^{(n)}(\hat{\mu}_{t}^{k,\underline{U}}, U_t))\\
        &\leq W_{2,\lambda}(\Phi(\mu_{t}^{k,\underline{U}}, U_t), \Phi(\hat{\mu}_{t}^{k,\underline{U}}, U_t)) + W_{2,\lambda}(\Phi(\hat{\mu}_{t}^{k,\underline{U}}, U_t), \Phi^{(n)}(\hat{\mu}_{t}^{k,\underline{U}}, U_t))\\
        &\leq L_\Phi W_{2,\lambda}(\mu_{t}^{k,\underline{U}}, \hat{\mu}_{t}^{k,\underline{U}}) + \frac{1}{n}
    \end{align*}
    holds for all $k\in\cK$. The last inequality is due to the Lipschitz property of $\Phi$ (Corollary \ref{cor:phi-lipschitz}) and the asymptotic stability of the state quantizer $Q_n$ (Lemma \ref{lemma:approximate-measure}). By induction, we obtain
    \begin{equation*}
        W_{2,\lambda}(\mu_{T}^{k,\underline{U}}, \hat{\mu}_{T}^{k,\underline{U}}) \leq \frac{2}{n}\sum_{s=0}^TL_\Phi^s.
    \end{equation*}
    Hence,
    \begin{equation*}
        |V(\underline{U}) - V_n(\underline{U})| \leq \frac{2}{n}\sqrt{\operatorname{diam}(S)^2 + \lambda}\sum_{s=0}^TL_\Phi^s.
    \end{equation*}
    We write $\alpha_n := \frac{2}{n}\leq \sqrt{\operatorname{diam}(S)^2 + \lambda}\sum_{s=0}^TL_\Phi^s$ for each $n\in\NN$ for simplicity.
    \medskip

    \noindent\textbf{Step-2:} The term $|V_n(\underline{U}^{(m),*}) - V^{(\ell,n,m)}(\underline{U}^{(m),*})|$ contributes the total error because of measure quantization. To compute a bound, we apply a similar machinery
\begin{align*}
    |V_n(\underline{U}^{(m),*}) - V^{(\ell,n,m)}(\underline{U}^{(m),*})| &= \bigg|\frac{1}{K}\sum_{k=1}^K W_{2,\lambda}(\bar{\mu}_{T}^{k,\underline{U}^{(\ell,n,m)*}}, \nu^k)^2\\
    &\hspace{10em}- W_{2,\lambda}(\hat{\mu}_{T}^{k,\underline{U}^{(\ell,n,m)*}}, R_\ell(\nu^k))^2\bigg|\\
    &\leq \frac{1}{K}\sum_{k=1}^K |W_{2,\lambda}(\bar{\mu}_{T}^{k,\underline{U}^{(\ell,n,m)*}}, \nu^k) - W_{2,\lambda}(\hat{\mu}_{T}^{k,\underline{U}^{(\ell,n,m)*}}, R_\ell(\nu^k))|\\
    &\qquad\qquad\cdot|W_{2,\lambda}(\bar{\mu}_{T}^{k,\underline{U}^{(\ell,n,m)*}}, \nu^k) + W_{2,\lambda}(\hat{\mu}_{T}^{k,\underline{U}^{(\ell,n,m)*}}, R_\ell(\nu^k))|.
\end{align*}
Since $\cP^{(\ell)}(\XX_n)$ can be embedded into $\cP(\XX)$ isomorphically, $\XX_n\subset\XX$, and $\UU_m\subset\UU$, we have
\begin{equation*}
    |W_{2,\lambda}(\bar{\mu}_{T}^{k,\underline{U}^{(\ell,n,m)*}}, \nu^k) + W_{2,\lambda}(\hat{\mu}_{T}^{k,\underline{U}^{(\ell,n,m)*}}, R_\ell(\nu^k))| \leq 2\sqrt{\operatorname{diam}(S)^2 + \lambda}.
\end{equation*}
The first term is decomposed into
\begin{align*}
    |W_{2,\lambda}(\bar{\mu}_{T}^{k,\underline{U}^{(\ell,n,m)*}}, \nu^k) - W_{2,\lambda}(\hat{\mu}_{T}^{k,\underline{U}^{(\ell,n,m)*}}, R_\ell(\nu^k))| &\leq |W_{2,\lambda}(\bar{\mu}_{T}^{k,\underline{U}^{(\ell,n,m)*}}, \nu^k) \\
    &\hspace{5em}- W_{2,\lambda}(\bar{\mu}_{T}^{k,\underline{U}^{(\ell,n,m)*}}, R_\ell(\nu^k))|\\
    &+|W_{2,\lambda}(\bar{\mu}_{T}^{k,\underline{U}^{(\ell,n,m)*}}, R_\ell(\nu^k)) \\
    &\hspace{5em}- W_{2,\lambda}(\hat{\mu}_{T}^{k,\underline{U}^{(\ell,n,m)*}}, R_\ell(\nu^k))|.
\end{align*}
Using reverse triangle inequality and Proposition \ref{prop:resnik-quantizer-error}, it follows that
\begin{equation*}
    |W_{2,\lambda}(\bar{\mu}_{T}^{k,\underline{U}^{(\ell,n,m)*}}, \nu^k) - W_{2,\lambda}(\bar{\mu}_{T}^{k,\underline{U}^{(\ell,n,m)*}}, R_\ell(\nu^k))| \leq W_{2,\lambda}(\nu^k, R_\ell(\nu^k)) \leq \rho_\ell \xrightarrow{\ell\to\infty} 0^+.
\end{equation*}
We treat the other term by using reverse triangle inequality and a reasoning similar to the previous case
\begin{equation*}
    |W_{2,\lambda}(\bar{\mu}_{T}^{k,\underline{U}^{(\ell,n,m)*}}, R_\ell(\nu^k)) - W_{2,\lambda}(\hat{\mu}_{T}^{k,\underline{U}^{(\ell,n,m)*}}, R_\ell(\nu^k))| \leq W_{2,\lambda}(\bar{\mu}_{T}^{k,\underline{U}^{(\ell,n,m)*}}, \hat{\mu}_{T}^{k,\underline{U}^{(\ell,n,m)*}})
\end{equation*}
for which we have the following recursive relation
\begin{align*}
    W_{2,\lambda}(\bar{\mu}_{t+1}^{k,\underline{U}^{(\ell,n,m)*}},\; &\hat{\mu}_{t+1}^{k,\underline{U}^{(\ell,n,m)*}})\\
    &= W_{2,\lambda}(\Phi^{(n)}(\bar{\mu}_{t}^{k,\underline{U}^{(\ell,n,m)*}}, U_t^{(\ell,n,m)*}), R_\ell(\Phi^{(n)}(\hat{\mu}_{t}^{k,\underline{U}^{(\ell,n,m)*}}, U_t^{(\ell,n,m)*})))\\
    &\leq W_{2,\lambda}(\Phi^{(n)}(\bar{\mu}_{t}^{k,\underline{U}^{(\ell,n,m)*}}, U_t^{(\ell,n,m)*}), \Phi^{(n)}(\hat{\mu}_{t}^{k,\underline{U}^{(\ell,n,m)*}}, U_t^{(\ell,n,m)*}))\\
    &\qquad+ W_{2,\lambda}(\Phi^{(n)}(\hat{\mu}_{t}^{k,\underline{U}^{(\ell,n,m)*}}, U_t^{(\ell,n,m)*}), R_\ell(\Phi^{(n)}(\hat{\mu}_{t}^{k,\underline{U}^{(\ell,n,m)*}}, U_t^{(\ell,n,m)*})))\\
    &\leq L_\Phi W_{2,\lambda}(\bar{\mu}_{t}^{k,\underline{U}^{(\ell,n,m)*}}, \hat{\mu}_{t}^{k,\underline{U}^{(\ell,n,m)*}}) + \rho_\ell
\end{align*}
for all $t\in\cT$ and $k\in\cK$, where { we use Proposition \ref{prop:resnik-quantizer-error} and Corollary \ref{cor:phi-lipschitz}}. Consequently, we have
\begin{equation*}
    |V_n(\underline{U}^{(m),*}) - V^{(\ell,n,m)}(\underline{U}^{(m),*})| \leq 2\rho_\ell\sqrt{\operatorname{diam}(S)^2 + \lambda}\sum_{s=0}^TL_\Phi^s =: \beta_\ell.
\end{equation*}
\medskip

\noindent\textbf{Step-3:} We continue with the third estimate $|V^{(\ell,n,m)*} - V_n^*|$ in which we have two types of errors, namely the errors due to measure quantization and action quantization. Let $\underline{U}^{(\ell,n,m)*}\in\UU_m^T$ be optimal for $V^{(\ell,n,m)}$ and $k,\underline{U}^{*}\in\UU^T$ be optimal for $V_n^*$. Then we consider
\begin{equation*}
    |V^{(\ell,n,m)*} - V_n^*| \leq |V^{(\ell,n,m)*}-\inf_{\underline{U}^{(m)*}\in\UU_m^T}V_n(\underline{U}^{(m)*})| + |\inf_{\underline{U}^{(m)*}\in\UU_m^T}V_n(\underline{U}^{(m)*}) - V_n^*|.
\end{equation*}
The first summand on the majorizing side is dominated by $\sup_{\underline{U}^{(m)}\in\UU_m^T}|V^{(\ell,n,m)}(\underline{U}^{(m)}) - V_n(\underline{U}^{(m)})| \leq \beta_\ell$ by the preceding part involving the second term in the main decomposition. Define $V_{n,m}^* := \inf_{\underline{U}^{(m)\prime}\in\UU_m^T}V_n(\underline{U}^{(m)\prime})$. For any $r > 0$, pick $\underline{U}^r\in\UU^T$ so that $V_n(\underline{U}^r) \leq V_n^* + r$ exploiting the approximation property of infimum. Then, there is $\underline{U}^{(m),r}\in\UU_m^T$ such that $\max_{0\leq t\leq T-1}\|U_t^{(m),r} - U_t^r\|_\UU < \frac{1}{m}$. Let $\omega_n$ be the modulus of continuity of $V_n$, i.e. $\omega_n(\delta) := \sup\{|V_n(\underline{U}) - V_n(\underline{U}^\prime)| : \max_{0\leq t\leq T-1}\|U_t - U_t^\prime\| < \delta\}$, $\delta\in\overline{R}_+$. { Note that $V_n$ is a continuous function on a compact set, thereby, a uniformly continuous function on $\UU^T$, and in the special case $\UU_m^T$.} Then, we get
\begin{equation*}
    V_{n,m}^* \leq V_n(\underline{U}^{(m),r}) \leq V_n(\underline{U}^r) + \omega_n\left(\frac{1}{m}\right) \leq V_n^* + r + \omega_n\left(\frac{1}{m}\right).
\end{equation*}
Since $r>0$ is arbitrary, we get $0 \leq V_{n,m}^* - V_n^* \leq \omega_n\left(\frac{1}{m}\right)$. Moreover $\omega_n\left(\frac{1}{m}\right)\to 0$ as $n,m\to\infty$ since $V_n$, $n\in\NN$, is uniformly continuous and so is the limit function $V$. Thus, we get
\begin{equation*}
    |V^{(\ell,n,m)*} - V_n^*| \leq \beta_\ell + \omega_n\left(\frac{1}{m}\right) =: \xi_{\ell,n,m}\xrightarrow{\ell,n,m\to+\infty}0.
\end{equation*}

\noindent\textbf{Step-4:} Lastly, recall that we have $\sup_{\underline{U}\in\UU^T}|V(\underline{U}) - V_n(\underline{U})| \leq \alpha_n \xrightarrow{n\to+\infty}0$. Let $\underline{U}_n^*$ be optimal for $V_n$ and $\underline{U}^*$ be optimal for $V$. Then
\begin{equation*}
    V_n(\underline{U}^*) \leq V(\underline{U}^*) + \alpha_n = V^* + \alpha_n \implies V_n^* \leq V^* + \alpha_n
\end{equation*}
and
\begin{equation*}
    V(\underline{U}_n^*) \leq V_n(\underline{U}_n^*) + \alpha_n \leq V_n^* + \alpha_n \implies V^* \leq V_n^* + \alpha_n.
\end{equation*}
Combining both, we obtain $|V^* - V_n^*| \leq \alpha_n$.

\noindent\textbf{Step-5:}Finally, the original decomposition yields
\begin{equation*}
    |V(\underline{U}^{(m),*}) - V^*| \leq 2\alpha_n + \beta_\ell + \xi_{\ell,n,m}.
\end{equation*}
Defining $\varepsilon_{\ell,n,m} := 2\alpha_n + \beta_\ell + \xi_{\ell,n,m}$ for all $\ell,n,m\in\NN$ concludes this proof.\hfill$\blacksquare$

\section{Proofs of the Results in Section~\ref{sec:robustness}}

\subsection{Proof of Lemma~\ref{lem:value-function-is-continuous}}\label{proof:value-function-is-continuous}
In metric spaces, the continuity is characterized by sequential continuity. So, let $(P_j)_{j\in\NN}$ be a sequence of measures on $\cP(\cP(\XX)\times\cP(\XX))$ converging to $P\in\cP(\cP(\XX)\times\cP(\XX))$ in the weak\textsuperscript{$*$}-topology of $\cP(\cP(\XX)\times\cP(\XX))$, and let $(\underline{U}_j)_{j\in\NN}$ be a sequence of sequences of control actions in $\UU^T$, converging to $\underline{U}\in\UU^T$. Under the same assumptions, we know that $\Phi$ is uniformly continuous in each of its arguments, see the proof of Proposition \ref{prop:Continuity}-(ii).

    First, we fix $\underline{A}\in\UU^T$ and consider
    \begin{equation*}
        V(P_j, \underline{A}) = \int_{\cP(\XX)\times\cP(\XX)}P_j(\de\mu,\de\nu)W_{2,\lambda}(\Phi_T(\mu, \underline{A}), \nu)^2 \qquad \forall j\in\NN.
    \end{equation*}
    We know that $W_{2,\lambda}$ is Lipschitz continuous in each of its arguments, and, under Assumption \ref{assume:compactness}, it is bounded. Therefore, the cost function $\fc$ is a continuous and bounded function on its domain. By the definition of weak\textsuperscript{$*$}-convergence, we have
    \begin{equation*}
        V(P_j, \underline{A}) \to V(P, \underline{A})
    \end{equation*}
    and this convergence is uniform. Second, we fix $P\in\cP(\cP(\XX)\times\cP(\XX))$ and consider
    \begin{equation*}
        V(P, \underline{U}_j) = \int_{\cP(\XX)\times\cP(\XX)}P(\de\mu,\de\nu)W_{2,\lambda}(\Phi_T(\mu, \underline{U}_j), \nu)^2 \qquad \forall j\in\NN.
    \end{equation*}
    The function $\fc$ is bounded and measurable since we use Borel $\sigma$-fields and $\fc$ is continuous. Using dominated convergence theorem \cite[Theorem 2.4.5]{cohn2013measure},
    we get
    \begin{equation*}
        V(P, \underline{U}_j) \to V(P, \underline{U}).
    \end{equation*}
    Note that this convergence is uniform as well. Finally, since $V$ is uniformly convergent in both of its arguments, we have the joint continuity on its domain.\hfill$\blacksquare$

\subsection{Proof of Theorem~\ref{thm:robustness}}\label{proof:robustness}
For any $r\in\NN$, we can perform the following decomposition
    \begin{equation*}
        |V(\PP_r, \underline{U}_r^*)-V^*(\PP)| \leq |V(\PP, \underline{U}_r^*) - V(\PP_r,\underline{U}_r^*)| + |V(\PP_r,\underline{U}_r^*) - \underbrace{V(\PP, \underline{U}^*)}_{V^*(\PP)}|
    \end{equation*}
    The second term on the dominating side goes to zero as $r\to+\infty$ since the value function is continuous (Lemma \ref{lem:value-function-is-continuous}) and since the model is weak Feller (Corollary \ref{cor:transition-kernel}). For the first term, construct the family of functions
    \begin{equation*}
        \cG := \{\fc_{\underline{U}} : \underline{U} \in \UU^T\} \subset \cC(\cP(\XX)\times\cP(\XX)),
    \end{equation*}
    where $\fc_{\underline{U}}(\cdot) := \fc(\cdot, \underline{U})$. Since $\fc$ is continuous on a compact space, it is uniformly continuous. Meaning that for every $\epsilon > 0$, there exists $\delta > 0$ such that for all $\underline{U} \in \UU^T$ and all $P_1,P_2\in\cP(\cP(\XX)\times\cP(\XX))$, we have
    \begin{equation*}
        W_1(P_1, P_2) < \delta \implies |\fc(P_1, \underline{U}) - \fc(P_2, \underline{U})| < \epsilon.
    \end{equation*}
    Here, $W_1$ is, with an abuse of notation, the product $1$-Wasserstein distance on $\cP(\XX)\times\cP(\XX)$, which generates the weak\textsuperscript{$*$}-topology on the same space. However, the above expression can be equivalently written as for every $\epsilon > 0$, there exists $\delta > 0$ such that for all $\fc_{\underline{U}}\in\cG$ and all $P_1,P_2\in\cP(\XX)\times\cP(\XX)$, we have
    \begin{equation*}
        W_1(P_1, P_2) < \delta \implies |\fc_{\underline{U}}(P_1) - \fc_{\underline{U}}(P_2)| < \epsilon.
    \end{equation*}
    By definition, $\cG$ is equicontinuous. Moreover, by also using the continuity on compact space, we conclude that
    \begin{equation*}
        \sup_{\fc_{\underline{U}}\in\cG}\sup_{P\in\cP(\XX)\times\cP(\XX)}|\fc_{\underline{U}}(P)|=\sup_{\underline{U}\in\UU^T}\sup_{P\in\cP(\XX)\times\cP(\XX)}|\fc(P, \underline{U})| < +\infty.
    \end{equation*}
    Hence, $\cG$ is also uniformly bounded. By Arzelà-Ascoli theorem \cite[Theorem VI-3.8]{conway1994course}, it follows that $\cG$ is totally bounded in $\cC(\cP(\XX)\times\cP(\XX))$ with sup norm. Hence, for any $\epsilon > 0$, there exists $\underline{U}^1,\ldots,\underline{U}^M$ in $\UU^T$ such that for all $\underline{U}\in\UU^T$ there is $i\in\{1,\ldots, M\}$ satisfying $\|\fc_{\underline{U}} - \fc_{\underline{U}^i}\|_{\infty} < \epsilon$.

    Fix any $\underline{U}\in\UU^T$ and let $\underline{U}^i\in\{\underline{U}^j : j=1,\ldots, M\}$ be the closest one to $\underline{U}$ in sup-metric, then consider the following
    \begin{align*}
        |V(\PP, \underline{U}) - V(\PP_r,\underline{U})| &= \bigg|\int\fc_{\underline{U}}\de\PP - \int \fc_{\underline{U}}\de\PP_r\bigg|\\
        &\leq \underbrace{\bigg|\int\fc_{\underline{U}}\de\PP - \int\fc_{\underline{U}^i}\de\PP \bigg|}_{<\epsilon} + \bigg|\int\fc_{\underline{U}^i}\de\PP - \int\fc_{\underline{U}^i}\de\PP_r \bigg|\\
        &\qquad\qquad+ \underbrace{\bigg|\int\fc_{\underline{U}^i}\de\PP_r - \int\fc_{\underline{U}}\de\PP_r\bigg|}_{<\epsilon}\\
        &<2\epsilon + \bigg|\int\fc_{\underline{U}^i}\de\PP - \int\fc_{\underline{U}^i}\de\PP_r \bigg|.
    \end{align*}
    Then we have
    \begin{equation*}
        \sup_{\underline{U}\in\UU^T}\bigg|\int\fc_{\underline{U}}\de\PP - \int \fc_{\underline{U}}\de\PP_r\bigg| < 2\epsilon + \max_{1 \leq i\leq M}\bigg|\int\fc_{\underline{U}^i}\de\PP - \int\fc_{\underline{U}^i}\de\PP_r \bigg|.
    \end{equation*}
    By the uniform continuity of the value function, again Lemma \ref{lem:value-function-is-continuous}, the second term on the right hand side goes to zero as $r$ goes to infinity. From here, it follows that
    \begin{equation*}
        |V(\PP, \underline{U}_r^*) - V(\PP_r,\underline{U}_r^*)|\xrightarrow{r\to+\infty}0.
    \end{equation*}
    This concludes the proof of the statement using the decomposition given before.\hfill$\blacksquare$

\subsection{Proof of Corollary~\ref{cor:LimitPoint}}\label{proof:LimitPoint}

Since $\UU$ is compact by Assumption~\ref{assume:compactness}, then so is $\UU^T$. Any sequence in a compact space has a convergent subsequence. Write $(\underline{U}_{r_j}^*)_{j\in\NN}$ as a convergent subsequence of $(\underline{U}_r^*)_{r\in\NN}$. Let $\underline U^*$ be the limit. Then by Lemma~\ref{lem:value-function-is-continuous} and Theorem~\ref{thm:robustness},
\[
    V(\PP, \underline U^*) = \lim_{j\to\infty} V(\PP_{r_j}, \underline{U}_{r_j}^*) = \lim_{j\to\infty} V^*(\PP_{r_j}) = V^*(\PP),
\]
which completes this proof.\hfill$\blacksquare$

\acks{This research has been partially supported by TUBITAK through the 2211-E program, by the Turkish Academy of Sciences through the 2025 GEBIP Award, and by the Natural Sciences and Engineering Research Council of Canada (NSERC)}

\end{document}